\documentclass[11pt]{article}
\usepackage{times}
\usepackage{natbib}
\usepackage{amsmath,bm}
\usepackage{microtype}
\usepackage{url}
\usepackage{booktabs}
\usepackage{mathtools}
\usepackage{amsthm}
\usepackage{graphicx}
\usepackage{subcaption}
\usepackage{titletoc}
\usepackage{thmtools}
\usepackage{amssymb}
\usepackage[margin=1in]{geometry}
\usepackage[colorlinks=true, allcolors=blue, pagebackref]{hyperref}
\renewcommand*{\backrefalt}[4]{%
    \ifcase #1 \footnotesize{(not cited)}%
    \or        \footnotesize{(cited on page~#2)}%
    \else      \footnotesize{(cited on pages~#2)}%
    \fi}

\newtheorem{theorem}{Theorem}
\newtheorem{corollary}[theorem]{Corollary}
\newtheorem{lemma}[theorem]{Lemma}
\newtheorem{proposition}[theorem]{Proposition}

\newtheorem{remark}[theorem]{Remark}

\newcommand{\bb}[1]{\mathbb{#1}}
\renewcommand{\P}{\bb{P}}
\newcommand{\E}{\bb{E}}
\newcommand{\R}{\bb{R}}
\newcommand{\Var}{\mathrm{Var}}
\newcommand{\Cov}{\mathrm{Cov}}
\newcommand{\I}{\operatorname{I}}
\newcommand{\dd}{\,\mathrm{d}}
\newcommand{\Nml}{\mathcal{N}}
\newcommand{\eps}{\varepsilon}
\newcommand{\Ker}{\mathrm{Ker}}
\newcommand{\li}{\lambda_i}
\DeclareMathOperator{\Tr}{Tr}
\DeclarePairedDelimiter{\abs}{\lvert}{\rvert}
\DeclarePairedDelimiter{\Norm}{\lVert}{\rVert}

\title{Diffusion Removes Langevin's Conditioning Dependence:  \\ A Sharp Gaussian Analysis}

\author{Adam Perbost, Francis Bach, Pierre Marion \\
Inria, École normale supérieure -- PSL Research University\\
Paris, France\\
\texttt{$\{$adam.perbost,francis.bach,pierre.marion$\}$@inria.fr} \\
}

\begin{document}

\maketitle

\begin{abstract}
Despite their empirical success, why diffusion models overcome the bottlenecks of classical score-based samplers remains unclear. In this work, we leverage Gaussian distributions to isolate this phenomenon. We establish 2-Wasserstein convergence bounds for optimized hyperparameters, showing that diffusion processes achieve a sampling error of $O(\sqrt{d\lambda_{\max}}\log N/N)$, where $d$ is the dimension, $N$ the number of sampling steps, and $\lambda_{\max}$ the largest eigenvalue of the target covariance matrix. Unadjusted and underdamped Langevin dynamics suffer from an additional $\sqrt\kappa$ factor, where $\kappa$ is the condition number. These rates follow from spectral bounds which are sharp: we confirm them via matching first-order asymptotics as $N\rightarrow\infty$. Our analysis provides a rigorous characterization, in the Gaussian setting, of how time-dependent score trajectories remove condition-number dependence during sampling. By contrast, in the learning phase, we show that estimating the unnoised score by gradient descent leads to essentially the same estimator as estimating a noisy score, which suggests that the benefits of noising do not come from the learning phase.
\end{abstract}

\section{Introduction}
Since their introduction, diffusion models \citep{sohl2015deep,song2019generative,ho2020denoising,song2021scorebased} have achieved state-of-the-art performance across a wide range of domains \citep{dhariwal2021diffusion,kong2020diffwave,hoogeboom2022equivariant,watson2023novo}. Training diffusion models requires estimating the gradient of the log-density (the so-called \textit{score}) of a noisy version of the target. Then samples are produced by iterative denoising through a formula involving this score.
Taking a step back, iterative score-based sampling strategies long predate modern diffusion models, as epitomized by the celebrated Langevin algorithm \citep{pavliotis2014stochastic,roberts1996exponential} and its many variants such as underdamped Langevin \citep{leimkuhler2015molecular}.
The key factors explaining their success compared to those more classical algorithms are still unclear: what exact role does the noising process in diffusion models play, and how does it interact with the learning and sampling phases?

Our primary goal in this work is to prove a separation in sampling efficiency between diffusion models and Langevin dynamics in terms of their dependence on the condition number. This sheds light on the practical effectiveness of diffusion models on ill-conditioned distributions such as natural images \citep{Guth2022wave,field1987relations}. 
To do so, we consider the class of (centered) Gaussian targets. Naturally, sampling from a Gaussian distribution is not our primary concern here. Rather, we use this analytically tractable setting to derive sharp convergence rates accompanied by matching asymptotics, in order to highlight how diffusion models overcome limitations of other score-based sampling methods. As discussed below, deriving tight condition-number dependence for more general target classes such as log-concave distributions is a long-standing open problem \citep{cheng2018underdamped,dalalyan2020sampling}, which is outside the scope of the present paper. Additionally, we focus for most of the paper on the sampling phase, comparing the sampling convergence rates of diffusion, unadjusted Langevin algorithm (ULA) and underdamped Langevin dynamics (ULD). Assuming access to the true score yields a clean analysis and the separation between diffusion and Langevin, hence for simplicity we focus on this case. The connection with the estimation of the score in the learning phase is discussed in the latter part of our work.

\paragraph{Contributions.}
We establish upper bounds on the 2-Wasserstein error between the target distribution and the distribution sampled from the diffusion, ULA and ULD. Convergence rates are derived in Section~\ref{results} by optimizing these bounds with respect to hyperparameters (see Table~\ref{tab:summary_rates} for summary). The bounds for diffusion and for Langevin differ by a factor proportional to the square root of the condition number. Thanks to explicit computations of the Wasserstein errors, we confirm these rates via matching first-order asymptotics. 
By contrast, in the learning phase, we show in Section~\ref{learning} that estimating the unnoised score by gradient descent leads to essentially the same estimator as estimating a noisy score, the two differing only in how they regularize the directions unseen in the data, which suggests that the benefits of noising do not come from the learning phase.

\begin{table}[t]
\caption{Convergence rates and $\eps$-complexities for the Gaussian target. Rates are the orders of $W_2(\P_N,\P)$ as $N\to\infty$ for fixed $\Sigma$ (Theorems~\ref{thm:Langevin_rates} and~\ref{thm:diffusion_spectral_bound}). The $\eps$-complexity is the iteration budget $N$ needed to achieve $W_2(\P_N,\P)\leq\eps$, where $\tilde O$ hides logarithmic factors. Diffusion models remove both the conditioning factor $\sqrt{\kappa}$ of Langevin samplers and the burn-in of order $\kappa$.}

\vspace*{-.5cm}

\label{tab:summary_rates}
\begin{center}
\begin{tabular}{lcc}
\toprule
\textbf{Algorithm} & \textbf{Asymptotic equivalent} & \textbf{$\bm{\eps}$-complexity} \\
\midrule
ULA       & $\Theta\!\left(\lambda_{\max}\sqrt{\Tr(\Sigma^{-1})}\,\frac{\log N}{N}\right)$ & $\tilde{O}\!\left(\kappa + \frac{\sqrt{d\kappa\lambda_{\max}}}{\eps}\right)$ \\
ULD       & $\Theta\!\left(\lambda_{\max}\sqrt{\Tr(\Sigma^{-1})}\,\frac{\log N}{N}\right)$ & $\tilde{O}\!\left(\kappa+\frac{\sqrt{d\kappa\lambda_{\max}}}{\eps}\right)$ \\
Diffusion & $\Theta\!\left(\sqrt{\Tr\Sigma}\,\frac{\log N}{N}\right)$                    & $\tilde{O}\!\left(\frac{\sqrt{d\lambda_{\max}}}{\eps}\right)$ \\
\bottomrule
\end{tabular}
\end{center}
\end{table}

\paragraph{Related work.}
For ULA, a similar analysis of the two error terms was already performed, for example, by~\cite{durmus2019analysis,wibisono2018sampling} and specifically in the Gaussian case by \cite{pedregosa2023ulaq}. Building upon these works, we perform hyperparameter optimization and compute the asymptotic equivalent (rather than upper bounds).
Many works have studied ULD in the log-concave case, see, \textit{e.g.}, \cite{cheng2018underdamped,dalalyan2020sampling,leimkuhler2024contraction,kim2026analysis} for studies of the stochastic Euler scheme and Euler--Maruyama (EM) integrators. This line of work yields loose dependence on the condition number due to the bias term induced by non-vanishing step-sizes (see, \textit{e.g.}, Eq.~(11) in \cite{dalalyan2020sampling} with a $\kappa^{3/2}$ in the complexity bound compared to our $\sqrt{\kappa}$). A line of work has proved improved condition dependence for Gaussian targets  \citep{shen2019randomized,apers2022hamiltonian,gouraud2025hmc}, albeit with more complex integrators (splitting, random step-sizes, or Metropolis corrections). To our knowledge, the tight study of the bias of ULD for a Gaussian target and the EM integrator is a novel contribution.

For diffusion models, \cite{pierret2025diffusion} derive and numerically analyze error terms for Gaussian targets, but do not analytically derive optimal hyperparameters, discuss the dependence on the condition number, nor derive asymptotic equivalents. Closer are the works by \cite{hurault2025score} and \cite{hurault2026geometry}, which derive Wasserstein bounds for Gaussian data and discuss optimal hyperparameter choices depending on the covariance spectrum. A key difference with our work is that their analysis holds asymptotically for small step-sizes, whereas we make a fully non-asymptotic analysis (albeit in less generality), which enables us to capture the interaction between the step-size and other parameters, and to derive optimal step complexity with respect to our bounds.
Wasserstein bounds beyond the Gaussian case are given for instance by \cite{strasman2025an,beyler2025convergence}. 

\section{Sampling: diffusion removes the condition-number dependence}
\label{results}

\subsection{Problem setup}
\label{problem_setup}
We denote by $\P$ the target $d$-dimensional centered Gaussian distribution $\Nml(0,\Sigma)$. Let $(\lambda_i)_{1\leq i \leq d}$ denote the eigenvalues of the positive-definite covariance matrix $\Sigma$, and $\lambda_{\min}$ and $\lambda_{\max}$ be respectively the smallest and largest eigenvalues. The ratio 
\[
\kappa := \frac{\lambda_{\max}}{\lambda_{\min}}
\]
is the condition number of $\Sigma$, which measures the degree of anisotropy of $\P$. Ill-conditioning is a classical source of difficulty for Langevin-based sampling methods, as their convergence rates typically deteriorate with $\kappa$ \citep{cheng2018underdamped,dalalyan2020sampling}.

Let $p$ denote the Lebesgue density of $\P$. We define the score $s$ of the target distribution as
\[
s: x \mapsto \nabla \log p(x) = -\Sigma^{-1}x.
\]
We assume that the score is known exactly in this section; the learning problem is studied separately in Section~\ref{learning}.

We measure sampling error in the 2-Wasserstein distance, denoted by $W_2$. For two centered Gaussians whose covariance matrices $\Sigma_1$ and $\Sigma_2$ commute, the distance reduces to the Fréchet distance $\Norm{\Sigma_1^{1/2}-\Sigma_2^{1/2}}_F$, with $\Norm{\cdot}_F$ the Frobenius norm \citep{gelbrich1990formula}.
For each sampling scheme, we denote by $\P_k$ the law of its $k$-th iterate.

\paragraph{Time-homogeneous Langevin samplers.}
We first consider the Langevin process associated with~$\P$,
\begin{equation}
\label{eq:SDE_Langevin}
\dd X_t = s(X_t)\dd t+\sqrt{2}\dd W_t,\, t\in\R_+,\, X_0=0,
\end{equation}
where $(W_t)$ is a standard Brownian motion.

The unadjusted Langevin algorithm (ULA) \citep{dalalyan2017theoretical} is simply the Euler--Maruyama discretization of the SDE in \eqref{eq:SDE_Langevin}. Given a step-size $\eta$, an iteration budget $N$, and a sequence $(\xi_k)$ of independent standard Gaussian vectors in $\R^d$, it is defined by
\begin{equation}
\label{eq:ULA}
\begin{cases}
    X_0 = 0, \\
    \forall k \in \{0,\dots, N-1\},\,X_{k+1} = X_k + \eta s(X_k) + \sqrt{2\eta}\xi_k.
\end{cases}
\end{equation}

We next consider underdamped Langevin dynamics (ULD) \citep{cheng2018underdamped}, which introduces an auxiliary velocity variable into the Langevin dynamics with the aim of accelerating convergence:
\begin{equation}
\label{eq:SDE_UnderdampedLangevin}
\begin{cases}
    \dd X_t = V_t\dd t, \\
    \dd V_t = -\gamma V_t \dd t + s(X_t)\dd t + \sqrt{2\gamma} \dd W_t,
\end{cases}
t \in \R_+,
\end{equation}
with initialization $X_0 = V_0 = 0$, where $\gamma>0$ is the friction coefficient. We discretize these dynamics using the Euler--Maruyama scheme, yielding with the same notation as above
\begin{equation}
\label{eq:Euler_ULD}
\begin{cases}
    X_{k+1} = X_k + \eta V_k, \\
    V_{k+1} = (1-\eta\gamma) V_k +\eta s(X_k) + \sqrt{2\eta\gamma}\xi_k.
\end{cases}
\end{equation}

\paragraph{Time-inhomogeneous diffusion sampler.}
Diffusion-based sampling follows a different principle \citep{sohl2015deep,song2021scorebased}: rather than sampling directly with the score of the target distribution, it constructs a family of perturbed distributions and uses a time-dependent score along this family. Many conventions exist; we study the following one because it leads to telescoping cancellations in the computations, greatly simplifying our analysis (see Appendix~\ref{app:diffusion_closed_form}). Concretely, for a general (centered) target $\bb{Q}$, consider the forward noising process
\[
\begin{cases}
    \dd X_t = \dd W_t,\,t\in [0,T], \\
    X_0 = X \sim \bb{Q},
\end{cases}
\]
whose marginal $\bb{Q}_t$ at time $t$ has Lebesgue density $q_t$ and score $s_t := \nabla \log q_t$. The forward diffusion gradually smooths the target distribution, and $T$ is chosen large enough so that $\bb{Q}_T \approx \Nml(0,T\I)$, a distribution that is straightforward to sample from. Sampling proceeds by reversing this transformation,
\textit{i.e.}, by simulating the reverse-time SDE \citep{anderson1982reverse}
\begin{equation}
\label{eq:reverse_SDE}
\begin{cases}
    \dd Y_t = s_{T-t}(Y_t)\dd t + \dd W_t,\,t\in[0,T-\delta],
    \\ Y_0 \sim \Nml(0,T\I),
\end{cases}
\end{equation}
where $\Nml(0,T\I)$ replaces the intractable true marginal $\bb{Q}_T$ and $\delta>0$ is an early-stopping threshold, standard in the analysis of diffusion models, where it prevents a possible blow-up of the score as $t\to0$ \citep{debortoli2022convergence,benton2024nearly}.
In our Gaussian setting $\bb{Q} = \P$, all these quantities are explicit: the marginal at time $t$ is $\bb{Q}_t = \Nml(0,C_t)$ with $C_t=\Sigma+t\I$, and
\[
s_t: x \mapsto -C_t^{-1}x,
\]
whose Lipschitz constant $(\lambda_{\min}+t)^{-1}$ is bounded by $(\lambda_{\min}+\delta)^{-1}$ on $[\delta,T]$.

Since information about the target is concentrated near $t=T-\delta$, we discretize \eqref{eq:reverse_SDE} using a decreasing sequence of step-sizes $(\eta_k)$ instead of a constant $\eta$. We fix a decreasing grid $T=\tau_0>\tau_1>\cdots>\tau_N=\delta$ defining $\tau_k = \tau_{k-1}-\eta_k$ for $k\in\{1,\dots,N\}$. The diffusion algorithm we study is then
\begin{equation}
\label{eq:diffusion_algorithm}
\begin{cases}
Y_0 \sim \Nml(0,T\I), \\
\forall k= 1,\dots,N,\,\, Y_{k}= Y_{k-1}+\eta_ks_{\tau_{k-1}}(Y_{k-1})+\sqrt{\eta_k}\xi_k.
\end{cases}
\end{equation}
The main difference between diffusion and Langevin lies in the use of a time-dependent score $s_t$ in the former. Langevin methods only use the target score $s$, and therefore operate under a fixed geometry governed by $\kappa$. This leads, as we will show in Theorem~\ref{thm:Langevin_rates}, to a penalizing $\sqrt{\kappa}$ factor in the convergence rate. 
In contrast, diffusion progressively modifies the geometry of the score along the sampling trajectory. Indeed, note that the condition number associated with $s_t$ is
\[
\kappa_t = \frac{t+\lambda_{\max}}{t+\lambda_{\min}} \underset{t\rightarrow\infty}{\longrightarrow} 1,
\]
decreasing from $\kappa_0=\kappa$ to 1. Sampling therefore begins in a near-isotropic regime, and only gradually recovers the original conditioning of the target as the trajectory approaches $t=0$. This intuition is formalized in Theorem~\ref{thm:diffusion_spectral_bound}, showing that the sampling complexity of diffusion is indeed independent of $\kappa$.

\subsection{Langevin samplers: the condition-number dependence}
\label{results_Langevin}
Table~\ref{tab:summary_rates} summarizes our guarantees, with hyperparameters optimized for each method, and Figures~\ref{fig:convergence} and~\ref{fig:ratio-complexity} illustrate them. Proofs are given in the appendices.
\begin{figure}[t]
  \centering
  \includegraphics{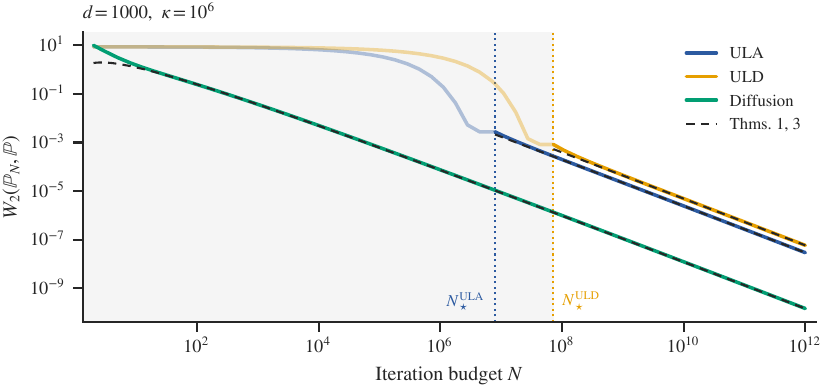}
  \caption{Sampling error $W_2(\P_N,\P)$ as a function of the iteration budget $N$ (log--log scale) on an ill-conditioned Gaussian target: $d=1000$, log-uniform spectrum $\li=\kappa^{-(i-1)/(d-1)}$ with $\kappa=10^6$. For this spectrum, $\lambda_{\max}\sqrt{\Tr(\Sigma^{-1})/\Tr\Sigma}=\sqrt\kappa$. Each point is a separate run: for every $N$, the sampler is tuned as in Theorems~\ref{thm:Langevin_rates} and~\ref{thm:diffusion_spectral_bound}, run for $N$ steps, and the error is evaluated in closed form. Dashed lines: asymptotic equivalents of the two theorems, drawn from $N_\star$ onwards for the Langevin samplers. Segments $N<N_\star$: the step-size of Theorem~\ref{thm:Langevin_rates} is inadmissible (Remark~\ref{rem:burnin}) and the largest admissible one is used instead, $\eta=\lambda_{\min}$ for ULA and $\eta=1/\gamma$ for ULD; see the discussion after Remark~\ref{rem:burnin}.
  }
    \label{fig:convergence}
\end{figure}
We first consider ULA and ULD. The ULA bound builds on the Gaussian analysis of \citet{pedregosa2023ulaq}. The convergence rates for ULA and ULD differ only by numerical constants.

\begin{figure}[t]
  \centering
  \begin{subfigure}[t]{0.48\textwidth}
    \centering
    \includegraphics[width=\linewidth]{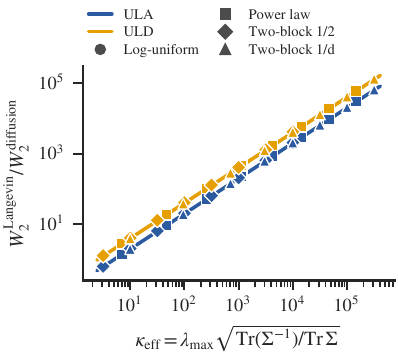}
    \caption{Error ratio versus $\kappa_{\mathrm{eff}}$, $N=10^{12}$.}
    \label{fig:ratios}
  \end{subfigure}\hfill
  \begin{subfigure}[t]{0.48\textwidth}
    \centering
    \includegraphics[width=\linewidth]{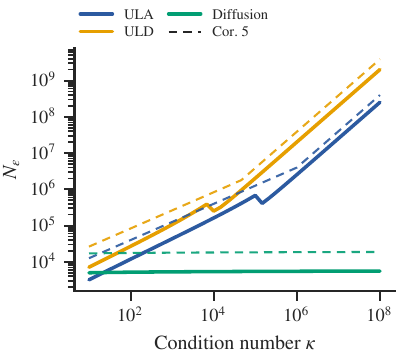}
    \caption{$\eps$-complexity versus $\kappa$, $\eps=10^{-3}\sqrt{\Tr\Sigma}$.}
    \label{fig:complexity}
  \end{subfigure}
    \caption{Tightness of the theory in dimension $d=1000$. \textbf{(a)} Ratio $W_2^{\mathrm{Langevin}}/W_2^{\mathrm{diffusion}}$ at budget $N=10^{12}$, as a function of $\kappa_{\mathrm{eff}}:=\lambda_{\max}\sqrt{\Tr(\Sigma^{-1})/\Tr\Sigma}$, for four spectral families with $\lambda_{\max}=1$ and $\kappa\in\{10,10^2,\dots,10^8\}$: ‘‘log-uniform'' ($\li=\kappa^{-(i-1)/(d-1)}$); ‘‘power law'' ($\li=i^{-a}$ with $a=\log\kappa/\log d$); ‘‘two-block $1/2$'' (half of the eigenvalues equal to $1$ and half to $1/\kappa$); ‘‘two-block $1/d$'' (one eigenvalue equal to $1$ and $d-1$ equal to $1/\kappa$). Langevin samplers are tuned as in Theorem~\ref{thm:Langevin_rates}; the diffusion uses the spectrum-agnostic schedule of Proposition~\ref{prop:diffusion_coarse_bound} ($T=\sqrt N$, $\delta = 1/N^2$). Lines are the quotients of the asymptotic equivalents, $\tfrac15\kappa_{\mathrm{eff}}$ (ULA) and $\tfrac25\kappa_{\mathrm{eff}}$ (ULD), with no fitted constant: all four families collapse onto them within $3.4\%$, consistent with a finite-$N$ correction of order $1/\log N \approx 3.6\%$.
    \textbf{(b)} Budget $N_\eps$ after which the error stays below $\eps=10^{-3}\sqrt{\Tr\Sigma}$, where $\sqrt{\Tr\Sigma}=W_2(\delta_0,\P)$ is the scale of the target, for the log-uniform spectrum. Dashed: the sufficient budgets of Corollary~\ref{cor:complexity} ($r=10$ in Remark~\ref{rem:variation_complexity} for the diffusion), within a factor $3$ of the observed ones. The diffusion budget does not depend on $\kappa$. The Langevin budgets follow the $\sqrt\kappa$ rate of Table~\ref{tab:summary_rates} as long as $\eps\lesssim\lambda_{\min}\sqrt{\Tr(\Sigma^{-1})}=\sqrt{\Tr\Sigma/\kappa}$; beyond, the burn-in $N_\star=\Theta(\kappa\log\kappa)$ of Remark~\ref{rem:burnin} takes over and the budget grows linearly in $\kappa$. Kinks on the ULA curve mark the switch between the two branches of its step-size.
    }
  \label{fig:ratio-complexity}
\end{figure}

\begin{restatable}[Langevin samplers]{theorem}{LangevinRates}
\label{thm:Langevin_rates}
Let $N\ge2$ be the iteration budget, and set the step-size $\eta$ (and the friction $\gamma$ for ULD) as
\begin{itemize}
    \item ULA: $\eta = \frac12\lambda_{\max}\frac{\log N}{N}$, provided $N\geq\frac12\kappa\log N$;        
    \item ULD: $\eta = \sqrt{2\lambda_{\max}\kappa}\,\frac{\log N}{N}$ and $\gamma = \sqrt{8/\lambda_{\min}}$, provided $N\geq4\kappa\log N$.
\end{itemize}
Then
\begin{equation*}
W_2(\P_N,\P)\leq\frac{\lambda_{\max}\sqrt{\Tr(\Sigma^{-1})}}{N}\left(C_1\log N+C_2\right),
\end{equation*}
with $C_1=\frac14$, $C_2 = \sqrt2$ for ULA and $C_1=\frac37$, $C_2=8$ for ULD. Furthermore, for fixed $\Sigma$, these bounds are tight up to a constant factor:
\[
W_2(\P_N,\P)\underset{N\to\infty}{\sim} C_3\,\lambda_{\max}\sqrt{\Tr(\Sigma^{-1})}\,\frac{\log N}{N},
\]
with $C_3 = \frac18$ for ULA and $C_3 = \frac14$ for ULD.
\end{restatable}

Both Langevin samplers deteriorate with the condition number through a $\sqrt{\kappa}$ factor, which can be made explicit by noting that $\lambda_{\max}\sqrt{\Tr(\Sigma^{-1})} = \sqrt{\lambda_{\max}} \sqrt{\sum_{i=1}^d \frac{\lambda_{\max}}{\li}} \in [\sqrt{\lambda_{\max} \kappa}, \sqrt{\lambda_{\max} d \kappa}]$.
The proof of this theorem, given in Appendix~\ref{app:thm1_proof}, exploits the Gaussian structure of the target, reducing the analysis to the eigenvalues of the covariance matrix, for which a closed-form expression can be obtained. This yields a Wasserstein error bound consisting of two terms: a discretization error, arising from the discretization of the continuous-time Langevin dynamics, and an early-stopping error, resulting from terminating the algorithm at time $N\eta$ rather than letting it run indefinitely. The resulting hyperparameter choices and convergence rates therefore reflect a trade-off between these two errors.

The ratio $\phi=\eta/\gamma$ of ULD equals the step-size $\eta$ of ULA. This is expected: for large friction, the position approximately follows the overdamped Langevin dynamics slowed down by a factor $\gamma$ \citep{pavliotis2014stochastic}, so that $\phi$ acts as an effective step-size. The small-friction regime yields the same rate (Corollary~\ref{cor:ULD_underdamped}): with the EM scheme, neither brings the acceleration expected from ULD.

\begin{remark}[Admissibility and burn-in]
\label{rem:burnin}
The conditions on $N$ in Theorem~\ref{thm:Langevin_rates} are exactly the requirements under which our analysis operates, namely $\eta\leq\lambda_{\min}$ for ULA and $\eta\gamma\leq1$ for ULD (Appendix~\ref{app:thm1_proof}). They amount to a burn-in phase: writing
\[
N_\star^{\mathrm{ULA}}:=\min\{n\geq3 : n\geq\tfrac12\kappa\log n\}
\quad\text{and}\quad
N_\star^{\mathrm{ULD}}:=\min\{n\geq3 : n\geq4\kappa\log n\},
\]
provided $N\ge3$, the condition on $N$ for ULA (resp.\ ULD) holds if and only if $N\geq N_\star^{\mathrm{ULA}}$ (resp.\ $N\geq N_\star^{\mathrm{ULD}}$), and both thresholds grow as $\Theta(\kappa\log\kappa)$. The guarantee thus applies only after $\Theta(\kappa\log\kappa)$ iterations. Ill-conditioning penalizes Langevin samplers twice: through the $\sqrt\kappa$ factor in the rate, and through the budget required before the guarantee applies. Neither penalty has a counterpart for diffusion: Theorem~\ref{thm:diffusion_spectral_bound} holds for every $N\geq1$.
\end{remark}

Figure~\ref{fig:convergence} illustrates this burn-in. Below $N_\star$, the step-size of Theorem~\ref{thm:Langevin_rates} violates the admissibility conditions of Remark~\ref{rem:burnin}. Since our bounds are convex (with respect to the variable $\eta$ for ULA, $\phi:=\eta/\gamma$ for ULD), their constrained minimizer is the largest admissible value, $\eta=\lambda_{\min}$ and $\eta\gamma=1$ respectively (Appendix~\ref{app:thm1_proof}). This step is dictated by the stiffest direction of the target, so the flattest ones barely move: the error stays at its initial value $W_2(\P_0,\P)=\sqrt{\Tr\Sigma}$ for $N\lesssim\kappa$, then decays like $e^{-cN/\kappa}$ ($c=2$ for ULA, $c=1/4$ for ULD) down to the discretization bias of that step, reached at $N\approx N_\star=\Theta(\kappa\log\kappa)$. Only then does the $\log N/N$ regime of Theorem~\ref{thm:Langevin_rates} begin, whereas the diffusion sampler follows its asymptote from $N\approx10^2$ onwards.

\subsection{Diffusion models: circumventing the condition-number dependence}
\label{results_diffusion}
In contrast, diffusion models follow the scores $s_t$, whose conditioning $\kappa_t$ tends to 1 at large times, and achieve rates that depend on neither $\kappa$ nor $\lambda_{\min}$.
\begin{restatable}[Diffusion sampler]{theorem}{DiffusionRates}
\label{thm:diffusion_spectral_bound}
    Consider an iteration budget $N\geq1$. Set horizon $T=(\Tr(\Sigma^5)/\Tr\Sigma)^{1/4}\sqrt{N}$, early-stopping threshold $\delta = \Tr\Sigma/(dN^2)$, and geometric time discretization grid $\tau_k = T(\delta/T)^{k/N}$ for $k \in \{0, \dots, N\}$. Then the diffusion algorithm \eqref{eq:diffusion_algorithm} yields
\[
W_2(\P_N,\P)\leq \frac{\sqrt{\Tr\Sigma}}{N}\Big(\tfrac58\,a_N\,\rho^{1/N}\log(\rho^{2/5}N)+1+\sqrt8\Big),
\quad \rho := \frac{d\lambda_{\max}}{\Tr\Sigma}\in[1,d],
\]
where $a_N := (1+4N^{-2})^{1/2}N^{5/(2N)}$ satisfies $a_N\le 2^{7/4}$ for all $N\ge1$ and $a_N\to1$. Moreover, this bound is asymptotically exact: under the same hyperparameters,
\begin{equation} \label{eq:spectral_equivalent}
    W_2(\P_N,\P)\underset{N\to\infty}{\sim}\frac58\sqrt{\Tr\Sigma}\,\frac{\log N}{N},
\end{equation}
so that, for fixed $\Sigma$, the ratio of the bound to the actual error tends to $1$.
\end{restatable}
The previous theorem gives a spectrum-adaptive bound. The simplified $O(\sqrt{\Tr\Sigma}\log N/N)$ rate is obtained when $d$ grows at most polynomially with $N$. The following proposition shows that this growth assumption can be removed at the price of a coarser dependence on the spectrum.
\begin{restatable}{proposition}{DiffusionRatesCoarse}
\label{prop:diffusion_coarse_bound}
    With the alternative choice $T = \lambda_{\max}\sqrt{N}$, early-stopping threshold $\delta = \lambda_{\max}/N^2$, and geometric time discretization grid $\tau_k = T(\delta/T)^{k/N}$ for $k \in \{0, \dots, N\}$, the diffusion algorithm~\eqref{eq:diffusion_algorithm} satisfies, for every $N\geq1$,
    \[
    W_2(\P_N,\P) \leq \frac{\sqrt{d\lambda_{\max}}}{N}\bigg(\frac5{2^{5/4}}\log N +\Big(\frac54\Big)^{3/2}+1\bigg).
    \]
\end{restatable}
On the one hand, Theorem~\ref{thm:diffusion_spectral_bound} provides our sharpest, spectrum-adaptive guarantee, with a prefactor governed by the effective scale $\sqrt{\Tr\Sigma}$. However, achieving this sharp rate requires choosing the horizon $T$ and early-stopping threshold $\delta$ using spectral information about $\Sigma$, and its simplified $O\big(\sqrt{\Tr\Sigma}\frac{\log N}{N}\big)$ form is stated under the growth condition $d \leq P(N)$ for some polynomial $P$. On the other hand, Proposition~\ref{prop:diffusion_coarse_bound} holds unconditionally in $d$, at the price of a coarser prefactor. Since $\Tr\Sigma \leq d\lambda_{\max}$ (with equality if and only if the target is isotropic), the coarse bound exceeds the spectral bound by a factor $\sqrt{\rho} \in [1, \sqrt{d}]$.
Importantly, this gap is merely a finite-$N$ phenomenon reflecting specific hyperparameter choices rather than a different asymptotic sampling regime. 
We show in Appendix~\ref{app:diffusion_equivalent_asymptotique} that any geometric schedule with $T \propto \sqrt{N}$ and $\delta \propto N^{-2}$ yields the same asymptotic equivalent \eqref{eq:spectral_equivalent}: the rate is robust to schedule tuning. This covers the parameter choices of both Theorem~\ref{thm:diffusion_spectral_bound} and Proposition~\ref{prop:diffusion_coarse_bound}.

As in Theorem~\ref{thm:Langevin_rates}, the proof of Theorem~\ref{thm:diffusion_spectral_bound} and Proposition~\ref{prop:diffusion_coarse_bound}, given in Appendix~\ref{app:thm2_proof}, exploits the Gaussian structure of the target distribution. Since the resulting distribution $\P_N$ is also Gaussian, the analysis reduces to the eigenvalues of its covariance matrix, which can be characterized explicitly. This leads to a Wasserstein error bound with three distinct contributions: a discretization error, an early-stopping error controlled by the threshold $\delta$, and an initialization error arising from starting the process at $\P_0 = \Nml(0,T\I)$ rather than at the true marginal $\Nml(0,\Sigma+T\I)$, which is not directly tractable. The resulting hyperparameter choices and convergence rates reflect a trade-off between these three sources of error. Consistently with the intuition of Section~\ref{problem_setup}, the geometric grid $\tau_k=T(\delta/T)^{k/N}$ takes coarse steps while $s_t$ remains close to isotropic, and only refines them as $t\to\delta$, when the anisotropy of the true score re-emerges. Such geometrically decreasing step-sizes near the data are standard in the analysis of diffusion models: the schedule of \citet{benton2024nearly}, for instance, combines constant steps far from the data with steps decaying geometrically as $t\to0$. We adopt the geometric grid over the whole horizon $[\delta,T]$, which makes the computations of Appendix~\ref{app:thm2_proof} easier.

The rates above describe the sampling error for a fixed iteration budget $N$. Sampling guarantees are, however, more commonly stated in the reverse direction in the literature: given a target accuracy, how many iterations are required? We follow this convention below.
\begin{restatable}[$\eps$-complexity]{corollary}{complexity}
\label{cor:complexity}
    Let $0<\eps\le\sqrt{\Tr\Sigma}$. In order for the sampling error to satisfy $W_2(\P_N,\P)\leq \eps$, it suffices
    \begin{itemize}
        \item for ULA with $\eta = \min\big(\lambda_{\min},\eps/\sqrt{\Tr(\Sigma^{-1})}\big)$, to take
        \[
        N\geq\frac\kappa2\max\bigg(1,\frac{\lambda_{\min}\sqrt{\Tr(\Sigma^{-1})}}{\eps}\bigg)\log\frac{\sqrt{8\Tr\Sigma}}{\eps},
        \]
        \item for ULD with $\eta = \min\big(\sqrt{\lambda_{\min}/8},\frac{7\sqrt2}{6}\eps/\sqrt{\lambda_{\min}\Tr(\Sigma^{-1})}\big)$ and $\gamma = \sqrt{8/\lambda_{\min}}$, to take
        \[
        N\geq 4\kappa\max\bigg(1, \frac{3}{14}\frac{\lambda_{\min}\sqrt{\Tr(\Sigma^{-1})}}\eps\bigg)\log\frac{16\sqrt{\Tr\Sigma}}\eps,
        \]
        \item for the diffusion algorithm with $\delta = \eps^2/(9d)$, $T=\sqrt{3/\eps}\Tr\big((\Sigma+\delta\I)^3\Sigma^2\big)^{1/4}$, and the geometric grid $\tau_k = T(\delta/T)^{k/N}$, to take
        \[
        N\ge\frac{3e}{4}\frac{\sqrt{\Tr\Sigma+4\eps^2/9}}\eps\log\frac T\delta.
        \]
    \end{itemize}
\end{restatable}
Using $\Tr\Sigma\le d\lambda_{\max}$ and $\Tr(\Sigma^{-1})\le d/\lambda_{\min}$, we retrieve the $\eps$-complexities of Table~\ref{tab:summary_rates}, \textit{i.e.}, $\tilde{O}(\kappa+\sqrt{d\kappa\lambda_{\max}}/\eps)$ for Langevin dynamics versus the condition-number-free $\tilde{O}(\sqrt{d\lambda_{\max}}/\eps)$ for diffusion. The spectral form is more informative. For the Langevin samplers, the $1/\eps$ rate only sets in once $\eps\le\lambda_{\min}\sqrt{\Tr(\Sigma^{-1})}$ for ULA and ${\eps\le\frac3{14}\lambda_{\min}\sqrt{\Tr(\Sigma^{-1})}}$ for ULD; above these thresholds, the budget is a burn-in of order $\kappa\log(\sqrt{\Tr\Sigma}/\eps)$, in line with Remark~\ref{rem:burnin}. The diffusion sampler, on the contrary, has no burn-in: its budget is proportional to $1/\eps$ for every $\eps\le\sqrt{\Tr\Sigma}$, up to the factor $\log(T/\delta)=O(\log(\sqrt{\Tr\Sigma}/\eps))$, which does not involve $\kappa$. Finally, the constants are better for ULA than for ULD; this reflects our proof technique rather than an intrinsic advantage, as we did not seek to optimize them.
\begin{remark}
\label{rem:variation_complexity}
The restriction $\eps\le\sqrt{\Tr\Sigma}$ is harmless: $\sqrt{\Tr\Sigma}=W_2(\delta_0,\P)$ is the error of the trivial output~$0$, so larger error tolerances require no iterations at all. Moreover, the constant $3e/4$ of the diffusion budget can be traded for a mild burn-in. We show in Appendix~\ref{app:cor_complexity_proof} that, with the same $\delta$, $T$ and grid, it suffices to take, for any $r\ge1$,
\[
N\ge r\max\bigg(1,\frac{3e^{1/r}}{4r}\frac{\sqrt{\Tr\Sigma+4\eps^2/9}}\eps\bigg)\log\frac T\delta.
\]
For $r=1$ the maximum is always attained by its second argument, which gives the corollary. As $r$ grows, the constant $3e^{1/r}/4$ decreases towards $3/4$, while the first argument, a $\kappa$-free burn-in of $r\log(T/\delta)$ iterations, prevails for $\eps>c_r\sqrt{\Tr\Sigma}$, with $c_r=\frac32\big(4r^2e^{-2/r}-1\big)^{-1/2}$.
\end{remark}

\section{Learning: the impact of noise on the implicit bias}
\label{learning}
The previous results assume the score of the target distribution---or, for diffusion, its noisy versions $s_t$---to be known exactly. In practice, these scores must first be learned from (noisy) samples.
The goal of this section is to study whether the noise introduced in diffusion models leads to a different score estimator. We are given a dataset $\mathcal{D} = \{x_i\}_{1\leq i \leq m}$ of $m$ independent realizations of $\P$ and denote by $\hat{\Sigma} = \frac1m\sum_{i=1}^m x_ix_i^\top$ the empirical covariance matrix. Since the score of a Gaussian distribution is linear, we consider the score model
\[
s_A : x \mapsto -Ax,
\]
where $A \in \R^{d\times d}$, which recovers the true score $s$ when $A=\Sigma^{-1}$. When $\hat\Sigma$ is invertible, which is almost surely the case if $m\ge d$, the score-matching loss \citep{hyvarinen2005estimation}
\begin{equation}
\label{eq:scorematching_loss}
L: A \mapsto \frac{1}{m}\sum_{i=1}^m \Big(\nabla\cdot s_A(x_i)
+\frac{1}{2}\Norm{s_A(x_i)}^2\Big),
\end{equation}
where $\Norm{\cdot}$ denotes the Euclidean norm, admits the unique minimizer $\hat\Sigma^{-1}$. In the overparameterized regime $m<d$, however, it has no minimizer as it diverges towards $-\infty$ along $\Ker\,\hat\Sigma$. In this regime, the direction of divergence followed by a given procedure is not determined by $L$ alone but also depends on the optimization algorithm and its initialization, which is usually referred to as an \textit{implicit bias} of the optimization algorithm \citep{gunasekar2018characterizing}. Our goal is to compare the implicit biases of two algorithms. First, we study gradient descent initialized at $A_0=0$ and write $A_K$ for the parameter obtained after $K\ge1$ gradient steps on \eqref{eq:scorematching_loss}. We then compare with the iterate $A^\sigma_K$ obtained by running $K$ steps of gradient descent, from the same initialization, on a diffusion-style noised objective. Fix a noise level $\sigma>0$ and let $\omega\sim\Nml(0,\sigma^2\I)$. The noised objective is
\[
L_\sigma : A \mapsto \frac{1}{m} \sum_{i = 1}^m \E_\omega\big[\nabla \cdot s_A(x_i+\omega)+\frac{1}{2}\Norm{s_A(x_i+\omega)}^2\big].
\]
As shown by \citet{vincent2011connection}, this coincides, up to an $A$-independent constant, with the standard denoising score-matching regression objective. Additionally, a direct computation shows that
\[
L_\sigma(A)=L(A)+\frac{\sigma^2}{2}\Norm{A}_F^2.
\]
Thus, it is equivalent to adding a ridge-type regularization term.

We show that, in this Gaussian setting, noise injection does not fundamentally change the learned score: both procedures recover the empirical covariance $\hat\Sigma$ on its range, up to a shift of order $\sigma^2$, and differ only in how they regularize its kernel. The key takeaway is that the benefit of the noising process identified in the sampling phase has no counterpart in the learning phase, at least in the Gaussian setting considered here.

Since $A$ plays the role of the precision matrix $\Sigma^{-1}$ in the score model $s_A(x)=-Ax$, Theorem~\ref{thm:implicit_bias} below characterizes, up to an exponentially decaying error, the covariance matrices estimated by each procedure---namely $A_K^{-1}$ and $(A_K^\sigma)^{-1}$. Let us denote by $\hat\lambda_{\max}$ the largest eigenvalue of $\hat\Sigma$, $\hat\lambda_{\min}$ its smallest eigenvalue, $\hat\lambda_{\min}^+$ its smallest positive eigenvalue and $\Pi_{\Ker\,\hat\Sigma}$ the orthogonal projector onto its kernel, of dimension $\max(0,d-m)$ almost surely.
\begin{restatable}[Score learning]{theorem}{LearningScore}
\label{thm:implicit_bias}
    Consider gradient descent initialized at $A_0=0$, with step-size $\alpha>0$, on both $L$ and $L_\sigma$. Let $K\ge1$.
    \begin{itemize}
        \item If $\alpha < \frac{2}{\hat\lambda_{\max}}$,
        the iterate $A_K$ is invertible and
        \[
        A_K^{-1} = \hat{\Sigma} + \frac1{K\alpha}\Pi_{\Ker\,\hat{\Sigma}}+O(\rho_1^K),
        \]
        with
        $
        \rho_1 = \max\big(\abs{1-\alpha\hat\lambda_{\min}^+},\abs{1-\alpha\hat\lambda_{\max}}\big)<1.
        $
        \item If $\alpha<\frac{2}{\hat\lambda_{\max}+\sigma^2}$,
        the iterate $A_K^\sigma$ is invertible and
        \[
        (A_K^\sigma)^{-1} = \hat{\Sigma}+\sigma^2\I + O(\rho_2^K),
        \]
        with
        $
        \rho_2 = \max\big(\abs{1-\alpha(\hat\lambda_{\min}+\sigma^2)},\abs{1-\alpha(\hat\lambda_{\max}+\sigma^2)}\big) < 1.
        $
    \end{itemize}
\end{restatable}
The distinction between the two procedures is clearest when $\hat\Sigma$ is singular, that is, almost surely when $m<d$. For the standard score-matching loss $L$, the resulting non-identifiability is resolved implicitly by gradient descent: up to an exponentially small optimization error, $A_K^{-1}$ coincides with $\hat{\Sigma}$ on its range, while a correction of order $1/(K\alpha)$ appears only in the kernel of $\hat\Sigma$. In contrast, the noised objective $L_\sigma$ induces an explicit ridge regularization, and this translates into a uniform covariance shift $\sigma^2 \I$: all directions are regularized by the same amount. 
Thus the two procedures recover a regularized empirical covariance matrix, but differ slightly in the regularization: the implicit bias of $L$ acts only in directions that are not identified by the data, whereas noise injection regularizes all directions indiscriminately. Yet, in order to recover the true covariance directions, one should take $\sigma^2 \ll \hat\lambda_{\min}^+$ and, respectively, $1/(K \alpha) \ll \hat\lambda_{\min}^+$. In this regime, both regularizations are essentially identical, meaning that the noising process does not help regularize the estimated score compared to learning the original score (with early stopping). We finally note that this constitutes an instance where the gradient descent path follows the explicit ridge regularization path, which is a structure appearing in many contexts \citep{Ali2020implicit}.

\section{Conclusion}

This work demonstrates an advantage of modeling the score at various levels of noise in the Gaussian setting: more efficient sampling by gradually adding anisotropy. This benefit is specific to sampling: in the same setting, the score learned with noise coincides, up to the regularization of the empirical nullspace, with the one learned without. Our sampling results take the form of scaling laws, both for the sampling error and for the hyperparameters that achieve it. We emphasize that, despite its simplicity, the Gaussian framework enables a rigorous proof by characterizing the exact dependence on the condition number. 
Weaker assumptions, \textit{e.g.}, log-concavity, lead to loose dependence on the condition number, which is insufficient for our purpose.

Natural directions for future work include incorporating the score estimation error into the sampling analysis and extending the comparison beyond the Gaussian setting.

\section*{Acknowledgements}
This work has received support
from the French government, managed by the National Research Agency, under the France 2030
program with the reference “PR[AI]RIE-PSAI” (ANR-23-IACL-0008), and through the Géné-Pi project, as part of the PEPR Artificial Intelligence programme (ANR-25-PEIA-0006).

\bibliographystyle{plainnat}
\bibliography{biblio}

\clearpage
\appendix

\counterwithin{theorem}{section}
\renewcommand{\theHtheorem}{\thesection.\arabic{theorem}}
\renewcommand{\thecorollary}{\thetheorem}
\renewcommand{\thelemma}{\thetheorem}
\renewcommand{\theproposition}{\thetheorem}
\renewcommand{\thedefinition}{\thetheorem}
\renewcommand{\theexample}{\thetheorem}
\renewcommand{\theremark}{\thetheorem}
\renewcommand{\theHcorollary}{\theHtheorem}
\renewcommand{\theHlemma}{\theHtheorem}
\renewcommand{\theHproposition}{\theHtheorem}
\renewcommand{\theHdefinition}{\theHtheorem}
\renewcommand{\theHexample}{\theHtheorem}
\renewcommand{\theHremark}{\theHtheorem}

\startcontents[sections]
\printcontents[sections]{l}{1}{\section*{Table of Contents}\setcounter{tocdepth}{2}}

\section{Preliminaries}
\label{app:prelim}
We start with a lemma that serves for all three algorithms. Throughout, the constant random vector $0$ is regarded as a degenerate centered Gaussian vector with a vanishing covariance matrix; this convention lets us handle the zero initializations of ULA and ULD without special treatment.
\begin{lemma}
\label{lem:gaussian_recursion}
    Let $P$ be an orthogonal matrix diagonalizing $\Sigma$, \textit{i.e.}, $\Sigma = P\Lambda P^\top$ with $\Lambda = \mathrm{diag}(\lambda_1,\dots,\lambda_d)$, and let $\mathcal{A}(\Sigma)$ denote the set of matrices that are diagonal in the eigenbasis of $\Sigma$:
    \[
    \mathcal{A}(\Sigma) = \{PDP^\top,\ D\textrm{ diagonal}\}.
    \]
For $M\in\mathcal{A}(\Sigma)$, we write $[M]_i$ for its $i$-th diagonal entry in this basis.
Let $(A_k^{xx},A_k^{xv},A_k^{vx},A_k^{vv},B_k^x,B_k^v)_{k\ge0}$ be elements of $\mathcal{A}(\Sigma)$, $(\xi_k)$ a sequence of independent standard Gaussian vectors in $\R^d$, and $(X_0,V_0)$ a centered, possibly degenerate, Gaussian pair, independent of $(\xi_k)$, whose covariance blocks $\E[X_0X_0^\top]$, $\E[V_0V_0^\top]$ and $\E[X_0V_0^\top]$ all belong to $\mathcal{A}(\Sigma)$. Consider the process $(X_k,V_k)$ defined by
\begin{equation}
\label{eq:Gaussian_system}
\forall k \ge0,
\begin{cases}
    X_{k+1} = A_k^{xx}X_k+A_k^{xv}V_k+B_k^x\xi_k, \\
    V_{k+1} = A_k^{vx}X_k+A_k^{vv}V_k+B_k^v\xi_k,
\end{cases}
\end{equation}
and set $\Sigma_k:=\E[X_kX_k^\top]$, $S_k:=\E[V_kV_k^\top]$ and $\Gamma_k:=\E[X_kV_k^\top]$. Finally, for $1\le i\le d$, denote by $x^{(i)}_k$, $v_k^{(i)}$ and $\zeta_k^{(i)}$ the $i$-th coordinates of $P^\top X_k$, $P^\top V_k$ and $P^\top \xi_k$ respectively. Then:
\begin{itemize}
    \item[(i)] $(X_k,V_k)$ is a centered Gaussian vector, and $\Sigma_k$, $S_k$ and $\Gamma_k$ belong to $\mathcal{A}(\Sigma)$.
    \item[(ii)] The $d$ pairs $(x_k^{(i)},v_k^{(i)})$, $1\le i\le d$, are mutually independent, and each of them follows the scalar recursion
    \[
    \forall k \ge0,
    \begin{cases}
        x_{k+1}^{(i)} = [A_k^{xx}]_ix_k^{(i)} + [A_k^{xv}]_iv_k^{(i)}+[B_k^{x}]_i\zeta_k^{(i)}, \\
        v_{k+1}^{(i)} = [A_k^{vx}]_ix_k^{(i)} + [A_k^{vv}]_iv_k^{(i)} + [B_k^v]_i\zeta_k^{(i)},
    \end{cases}
    \]
    where the $(\zeta_k^{(i)})_{k,i}$ are mutually independent standard one-dimensional Gaussian variables, independent of $(X_0,V_0)$.
    \item[(iii)] The eigenvalues $\lambda^{(k)}_i$ of $\Sigma_k$ are the variances of the variables $x_k^{(i)}$:
    \[
    \lambda_i^{(k)} = \Var(x_k^{(i)}).
    \]
\end{itemize}
\end{lemma}
 
\begin{remark}
\label{rem:identification_algos}
    ULA \eqref{eq:ULA}, ULD \eqref{eq:Euler_ULD} and the diffusion sampler \eqref{eq:diffusion_algorithm} are all instances of the process $(X_k,V_k)$ above. For ULD the identification is immediate: $A_k^{xx} = \I$, $A_k^{xv} = \eta\I$, $B_k^x=0$, $A_k^{vx} = -\eta\Sigma^{-1}$, $A_k^{vv} = (1-\eta\gamma)\I$ and $B_k^v = \sqrt{2\eta\gamma}\I$. ULA and the diffusion have no velocity component, but they fit the same framework upon taking $V_0 = 0$ and $A_k^{xv} = A_k^{vx} = A_k^{vv} = B_k^v = 0$. ULA then corresponds to $A_k^{xx}=\I-\eta\Sigma^{-1}$ and $B_k^x=\sqrt{2\eta}\I$, and the diffusion to $A_k^{xx}=\I-\eta_{k+1}C^{-1}_{\tau_k}$ and $B_k^x = \sqrt{\eta_{k+1}}\I$, the covariance $T\I$ of its initialization $\Nml(0,T\I)$ being indeed in $\mathcal{A}(\Sigma)$.
\end{remark}
 
\begin{proof}[Proof of Lemma~\ref{lem:gaussian_recursion}]
    We first prove (ii). For every $k\ge0$, set $\tilde{X}_k:=P^\top X_k$, $\tilde{V}_k := P^\top V_k$ and $\tilde{\xi}_k:=P^\top \xi_k$. Since $P$ is orthogonal, $(\tilde{\xi}_k)_{k\ge0}$ is again a sequence of independent standard Gaussian vectors, independent of $(X_0,V_0)$; hence the family $(\zeta_k^{(i)})_{k\ge0,\,1\le i\le d}$ consists of mutually independent $\Nml(0,1)$ variables, independent of $(X_0,V_0)$, which is the last claim of (ii). Multiplying \eqref{eq:Gaussian_system} by $P^\top$ and inserting $PP^\top = \I$ in front of $X_k$, $V_k$ and $\xi_k$ gives
    \[
    \begin{cases}
    \tilde{X}_{k+1} = P^\top A_k^{xx}P\tilde{X}_k + P^\top A_k^{xv}P\tilde{V}_k + P^\top B_k^xP\tilde{\xi}_k, \\
    \tilde{V}_{k+1} = P^\top A_k^{vx}P\tilde{X}_k+P^\top A_k^{vv}P\tilde{V}_k+P^\top B_k^vP\tilde{\xi}_k.
    \end{cases}
    \]
    By definition of $\mathcal{A}(\Sigma)$, every matrix in this system is diagonal, so reading it coordinate by coordinate yields exactly the announced recursion
    \[
    \begin{cases}
        x_{k+1}^{(i)} = [A_k^{xx}]_ix_k^{(i)} + [A_k^{xv}]_iv_k^{(i)}+[B_k^{x}]_i\zeta_k^{(i)}, \\
        v_{k+1}^{(i)} = [A_k^{vx}]_ix_k^{(i)} + [A_k^{vv}]_iv_k^{(i)} + [B_k^v]_i\zeta_k^{(i)},
    \end{cases}
    \quad 1\le i\le d.
    \]
    It remains to show that the $d$ pairs $(x_k^{(i)},v_k^{(i)})$ are independent, which we do by induction on $k$.
    
    \textit{Base case.} $(\tilde{X}_0,\tilde{V}_0)$ is a centered Gaussian vector, as $(X_0,V_0)$ is, and its covariance matrix
    \[
    \begin{bmatrix}
        P^\top\Sigma_0P & P^\top\Gamma_0P \\
        P^\top\Gamma_0^\top P & P^\top S_0P
    \end{bmatrix}
    \]
    has four diagonal blocks, since $\Sigma_0$, $S_0$ and $\Gamma_0$ belong to $\mathcal{A}(\Sigma)$. The pairs $(x_0^{(i)},v_0^{(i)})_{1\le i\le d}$ are therefore uncorrelated and, being jointly Gaussian, independent.
 
    \textit{Induction step.} Assume the pairs $(x_k^{(i)},v_k^{(i)})_{1\le i\le d}$ are independent for some $k\ge0$. The vector $\xi_k$ is independent of $X_0$, $V_0$ and $(\xi_j)_{j<k}$, of which $(X_k,V_k)$ is a deterministic function; hence $\xi_k$ is independent of $(X_k,V_k)$, and $\tilde{\xi}_k$ of $(\tilde{X}_k,\tilde{V}_k)$. As the coordinates $\zeta_k^{(1)},\dots,\zeta_k^{(d)}$ of $\tilde{\xi}_k$ are themselves independent, the $d$ triplets $(x_k^{(i)},v_k^{(i)},\zeta_k^{(i)})$ are mutually independent. Each pair $(x_{k+1}^{(i)},v_{k+1}^{(i)})$ is a deterministic function of the $i$-th triplet, so the pairs at step $k+1$ are independent as well. This completes the induction, and the proof of (ii).
 
    We now turn to (i). The vector $(X_{k+1},V_{k+1})$ is a linear function of $(X_k,V_k,\xi_k)$. If $(X_k,V_k)$ is a centered Gaussian vector, so is $(X_k,V_k,\xi_k)$, because $\xi_k$ is centered Gaussian and independent of $(X_k,V_k)$ (see the proof of (ii)); hence $(X_{k+1},V_{k+1})$ is centered Gaussian too, and the claim follows by induction from the assumption on $(X_0,V_0)$. As for the covariance matrices, $\Sigma_k$ belongs to $\mathcal{A}(\Sigma)$ if and only if $P^\top\Sigma_kP$ is diagonal. But $P^\top\Sigma_kP$ is the covariance matrix of $\tilde{X}_k$, whose $(i,j)$ entry is $\Cov(x_k^{(i)},x_k^{(j)})$, which vanishes for $i\neq j$ by (ii). Thus $\Sigma_k\in\mathcal{A}(\Sigma)$, and the same argument applies to $S_k$ and $\Gamma_k$.
 
    Finally, (iii) is a direct consequence of the above: the eigenvalues of $\Sigma_k$ are the diagonal entries of $P^\top\Sigma_kP$, namely the variances $\Var(x_k^{(i)})$.
\end{proof}
 
We also rely on the following elementary fact, which expresses the Wasserstein distance between two centered Gaussians whose covariances are diagonal in the same orthonormal basis in terms of their eigenvalues.
 
\begin{lemma}
\label{lem:W2_gaussian}
Let $\Sigma_1,\Sigma_2\in\mathcal{A}(\Sigma)$ be positive semi-definite, with diagonal entries (hence eigenvalues) $(a_i)_{1\le i\le d}$ and $(b_i)_{1\le i\le d}$ in the basis $P$. Then
\[
W_2\big(\Nml(0,\Sigma_1),\Nml(0,\Sigma_2)\big) = \left(\sum_{i=1}^d\Big(\sqrt{a_i}-\sqrt{b_i}\Big)^2\right)^{1/2}.
\]
\end{lemma}
\begin{proof}
Since $\Sigma_1$ and $\Sigma_2$ belong to $\mathcal{A}(\Sigma)$, they commute, and the formula recalled in Section~\ref{problem_setup} applies \citep{gelbrich1990formula}:
\[
W_2(\Nml(0,\Sigma_1),\Nml(0,\Sigma_2)) = \|\Sigma_1^{1/2}-\Sigma_2^{1/2}\|_F .
\]
Now $\Sigma_1^{1/2}=P\,\mathrm{diag}(\sqrt{a_1},\dots,\sqrt{a_d})\,P^\top$ and $\Sigma_2^{1/2}=P\,\mathrm{diag}(\sqrt{b_1},\dots,\sqrt{b_d})\,P^\top$, and the Frobenius norm is invariant under orthogonal conjugation, whence the result.
\end{proof}
 
Finally, the following elementary lemma is what we use to optimize the bounds of Propositions~\ref{prop:ULA_uniform_bound} and~\ref{prop:ULD_uniform_bound}.
 
\begin{lemma}
\label{lem:optim_convexe}
    Let $A$, $B$, $C$ and $D$ be positive constants, and let
    \[
    \begin{array}{rcl}
    f :\,(0,D] &\longrightarrow& \R, \\
    x &\longmapsto& Ae^{-Bx}+Cx.
    \end{array}
    \]
    Then $f$ is strictly convex and, provided $0<\log\frac{AB}C\le BD$, it admits a unique minimizer,
    \[
    x_* = \frac1B\log\frac{AB}C,
    \quad\text{with}\quad
    f(x_*) = \frac{C}B\bigg(1+\log\frac{AB}C\bigg).
    \]
\end{lemma}
 
\begin{proof}
    Strict convexity follows from $f''>0$ on $(0,D]$. Solving $f'(x)=0$ gives $x_*$, the condition $0<\log\frac{AB}C\le BD$ is exactly what ensures $0<x_*\le D$, and the value of $f(x_*)$ follows by substitution.
\end{proof}

\section{Proof of Theorem~\ref{thm:Langevin_rates} (Langevin samplers)}
\label{app:thm1_proof}
\subsection{Overview}
This section is devoted to the proof of the following theorem, stated in Section~\ref{results_Langevin}.

\LangevinRates*

\paragraph{Error decomposition.}
ULA and ULD define time-homogeneous Markov chains. Under stability conditions on the hyperparameters, established in step~(i) below, these chains converge, not to the target $\P$, but to an invariant law $\P_\infty$ that depends on the step-size $\eta$ (and on the friction $\gamma$ for ULD). The triangle inequality
\begin{equation}
\label{eq:decomp_W2}
W_2(\P_N,\P)\le W_2(\P_N,\P_\infty) + W_2(\P_\infty,\P)
\end{equation}
separates two contributions. The first term is the early-stopping error: after $N$ iterations, the chain has not yet reached its invariant law. The second is the discretization error: the bias of $\P_\infty$ with respect to $\P$, introduced by the Euler--Maruyama scheme.

\paragraph{Reduction to eigenvalues.}
Lemma~\ref{lem:gaussian_recursion} applies to both algorithms (Remark~\ref{rem:identification_algos}): each $\P_N$ is a centered Gaussian whose covariance $\Sigma_N$ lies in $\mathcal{A}(\Sigma)$, and so does the covariance $\Sigma_\infty$ of $\P_\infty$, the limit of $(\Sigma_N)$. By Lemma~\ref{lem:W2_gaussian}, both terms in \eqref{eq:decomp_W2} then depend only on the eigenvalues $\li^{(N)}$ of $\Sigma_N$, $\li^\infty$ of $\Sigma_\infty$ and $\li$ of $\Sigma$:
\begin{equation}
\label{eq:W2s_ULA_generic}
W_2^2(\P_N,\P_\infty) = \sum_{i=1}^d\Big(\sqrt{\li^{(N)}} - \sqrt{\li^\infty}\Big)^2,
\quad
W_2^2(\P_\infty,\P) = \sum_{i=1}^d\Big(\sqrt{\li^{\infty}} - \sqrt{\li}\Big)^2.
\end{equation}
This is the payoff of the Gaussian setting: the analysis reduces to real sequences for which closed forms are available.

Throughout, a choice of hyperparameters is called \emph{admissible} if it lies in the stability range of the chain, shrunk by numerical constants that give the estimates some slack, keeping the denominators involved bounded away from zero. Remarks~\ref{rem:ULA_convenience} and~\ref{rem:ULD_stricter} make the gap with the exact stability conditions explicit.

\paragraph{Outline.}
Sections~\ref{app:ULA_proof} (ULA) and~\ref{app:ULD_proof} (ULD) follow the same four steps.
\begin{itemize}
    \item[(i)] \emph{Closed form.} Compute the eigenvalues $\li^{(N)}$ and $\li^\infty$ from the scalar recursions of Lemma~\ref{lem:gaussian_recursion}, and identify the stability conditions under which $(\Sigma_N)$ converges.
    \item[(ii)] \emph{Parametric bound.} Bound each of the two terms in \eqref{eq:decomp_W2} and sum: this gives a bound on $W_2(\P_N,\P)$ valid over the whole admissible range identified in (i) (Proposition~\ref{prop:ULA_uniform_bound} for ULA, Proposition~\ref{prop:ULD_uniform_bound} for ULD).
    \item[(iii)] \emph{Choice of hyperparameters.} Minimize, up to a numerical constant, the bound obtained in (ii), and check that the selected values satisfy the stability conditions as soon as $N\ge N_\star$.
    \item[(iv)] \emph{Asymptotic equivalent.} For the hyperparameters of step~(iii), derive an exact equivalent of $W_2(\P_N,\P)$ as $N\to\infty$ from the closed forms of step~(i). This equivalent has the same order $\log N/N$ and the same spectral prefactor as the bound of step~(iii), up to a numerical constant: the bound is of the right order of magnitude.
\end{itemize}

\subsection{Unadjusted Langevin algorithm}
\label{app:ULA_proof}
The goal of this section is to prove the following proposition, which bounds $W_2(\P_N,\P)$ for every small enough step-size $\eta$. To lighten notation, we drop the index $i$ whenever we argue at a fixed eigendirection $i$.
\begin{proposition}[Parametric bound for ULA]
\label{prop:ULA_uniform_bound}
    For any step-size $0<\eta\le\lambda_{\min}$ and any $N\ge0$, the error of ULA satisfies
    \[
    W_2(\P_N,\P)\le \sqrt{2\Tr\Sigma}\,e^{-2\eta N/\lambda_{\max}}+\frac12\sqrt{\Tr(\Sigma^{-1})}\,\eta.
    \]
\end{proposition}
The two terms of this bound come from bounding the early-stopping error and the discretization error respectively. The first decays exponentially in $N\eta$, the time at which the process is stopped, while the second grows linearly with the step-size $\eta$. The step-size selected in Theorem~\ref{thm:Langevin_rates} is the best compromise between these two competing effects.

\subsubsection{Closed form of the eigenvalues}
\label{app:ULA_closed_form}
We first establish a closed form for the eigenvalues $\li^{(N)}$ and their limit $\li^\infty$.
\begin{lemma}
\label{lem:ULA_closed_form}
    Assume $0<\eta<2\lambda_{\min}$ and set
    \[
    \li^\infty := \frac{\li}{1-\frac{\eta}{2}\li^{-1}}.
    \]
    Then, for every $N\ge0$,
    \[
    \li^{(N)} = \li^\infty\left(1-\Big(1-\frac\eta\li\Big)^{2N}\right),
    \]
    and in particular $\li^{(N)}\to\li^\infty$ as $N\to\infty$.
\end{lemma}

\begin{proof}
Fix a direction $i$ and consider the corresponding sequence $(\lambda^{(k)})$ of eigenvalues of $\Sigma_k$ (the index $i$ is dropped, as announced). By Lemma~\ref{lem:gaussian_recursion}(iii) and the identification of Remark~\ref{rem:identification_algos}, $\lambda^{(k)}$ is the variance of $x_k$, where $x_0 = 0$ and
\[
\forall k\ge0,\quad x_{k+1} = \left(1-\frac\eta\lambda\right)x_k + \sqrt{2\eta}\,\zeta_k,
\]
with $(\zeta_k)$ independent standard one-dimensional Gaussians. Since $x_k$ is a function of $(\zeta_j)_{j<k}$, it is independent of $\zeta_k$, and taking variances gives $\lambda^{(0)}=0$ and
\[
\forall k\ge0,\quad \lambda^{(k+1)} = \left(1-\frac\eta\lambda\right)^2\lambda^{(k)}+2\eta,
\]
a first-order affine recursion. Its ratio satisfies
\begin{equation}
\label{eq:raison_contractante}
\left(1-\frac\eta\lambda\right)^2<1,
\end{equation}
because $0<\eta<2\lambda_{\min}$ gives $0<\eta/\lambda<2$, \textit{i.e.}, $-1<1-\eta/\lambda<1$. In particular the ratio is not $1$, so the recursion has a unique fixed point
\[
\ell =\frac{2\eta}{1-\left(1-\eta\lambda^{-1}\right)^2} = \frac{2\eta}{2\eta\lambda^{-1}-\eta^2\lambda^{-2}} = \frac{\lambda}{1-\frac{\eta}{2}\lambda^{-1}} = \lambda^\infty,
\]
and
\[
\forall k\ge0,\quad \lambda^{(k)} = \left(1-\frac\eta\lambda\right)^{2k}(\lambda^{(0)}-\ell) +\ell = \lambda^\infty\left(1-\Big(1-\frac\eta\lambda\Big)^{2k}\right).
\]
Taking $k=N$ gives the closed form, and \eqref{eq:raison_contractante} gives the convergence $\lambda^{(N)}\to\lambda^\infty$. Observe that $\lambda^\infty\to\lambda$ as $\eta\to0$: the finer the Euler--Maruyama discretization, the smaller the bias.
\end{proof}

\subsubsection{Early-stopping error}
We now bound the early-stopping error of ULA.
\begin{lemma}
\label{lem:ULA_early_stopping_error}
Assume $0<\eta\le\lambda_{\min}$. Then, for every $N\ge0$,
\[
W_2(\P_N,\P_\infty)\le\sqrt{2\Tr\Sigma}\,e^{-2\eta N/\lambda_{\max}}.
\]
\end{lemma}

\begin{proof}
Starting from the first identity in \eqref{eq:W2s_ULA_generic} and writing $\sqrt a-\sqrt b = (a-b)/(\sqrt a+\sqrt b)$,
\begin{equation}
\label{eq:ULA_W2_maj_denominateur}
W_2^2(\P_N,\P_\infty) = \sum_{i=1}^d \Bigg(\frac{\li^\infty-\li^{(N)}}{\sqrt{\lambda_i^\infty}+\sqrt{\lambda_i^{(N)}}}\Bigg)^2 \leq \sum_{i=1}^d \frac{\big(\li^\infty-\li^{(N)}\big)^2}{\li^\infty}.
\end{equation}
Since $\lambda^{(N)}\to\lambda^\infty$, this step only costs a factor 2 asymptotically (on $W_2$ itself, once the square root is taken). By Lemma~\ref{lem:ULA_closed_form}, $\lambda^\infty-\lambda^{(N)} = \lambda^\infty\left(1-\frac\eta\lambda\right)^{2N}$, hence
\[
\frac{(\lambda^\infty-\lambda^{(N)})^2}{\lambda^\infty} = \lambda^\infty\left(1-\frac{\eta}{\lambda}\right)^{4N}.
\]
The assumption $\eta\le\lambda_{\min}$ gives both $0\le1-\frac\eta\lambda\le1-\frac{\eta}{\lambda_{\max}}$ and $1-\frac\eta2\lambda^{-1}\ge\frac12$, \textit{i.e.}, $0<\lambda^\infty\le2\lambda$, so that
\[
\frac{(\lambda^\infty-\lambda^{(N)})^2}{\lambda^\infty} \le 2\lambda\left(1-\frac{\eta}{\lambda_{\max}}\right)^{4N}.
\]
Summing over $i$ and plugging into \eqref{eq:ULA_W2_maj_denominateur} yields
\[
W_2(\P_N,\P_\infty)\le\sqrt{2\Tr\Sigma}\left(1-\frac{\eta}{\lambda_{\max}}\right)^{2N} \le \sqrt{2\Tr\Sigma}\,e^{-2\eta N/\lambda_{\max}},
\]
where the last inequality uses $1-u\le e^{-u}$ together with the positivity of $1-\eta/\lambda_{\max}$.
\end{proof}

\begin{remark}
\label{rem:ULA_convenience}
Lemma~\ref{lem:ULA_closed_form} only requires $\eta<2\lambda_{\min}$, which is what the convergence of $\lambda^{(N)}$ to $\lambda^\infty$ needs, whereas Lemma~\ref{lem:ULA_early_stopping_error} assumes $\eta\le\lambda_{\min}$. The latter is a convenience assumption: it keeps $1-\eta\lambda^{-1}$ nonnegative and, more importantly, keeps the denominator $1-\frac\eta2\lambda^{-1}$ bounded away from zero. Since $\eta$ is meant to go to $0$ with $N$ to reduce the discretization error, this is a mild restriction.
\end{remark}

\subsubsection{Discretization error}
We turn to the discretization error.
\begin{lemma}
\label{lem:ULA_discretization_error}
Assume $0<\eta\leq\lambda_{\min}$. Then
\[
W_2(\P_\infty,\P)\leq\frac12\sqrt{\Tr(\Sigma^{-1})}\,\eta.
\]
\end{lemma}

\begin{proof}
This time we start from the second identity in \eqref{eq:W2s_ULA_generic}:
\begin{equation}
\label{eq:ULA_W2_maj_denominateur2}
W_2^2(\P_\infty,\P) = \sum_{i=1}^d \left(\frac{\li^\infty-\li}{\sqrt{\li^\infty}+\sqrt{\li}}\right)^2 \leq \frac14 \sum_{i=1}^d\frac{\left(\li^\infty-\li\right)^2}{\li},
\end{equation}
where the inequality uses
\begin{equation}
\label{eq:ULA_biais}
\lambda^\infty>\lambda,
\end{equation}
a consequence of $1-\frac\eta2\lambda^{-1}<1$. By Lemma~\ref{lem:ULA_closed_form},
\[
\lambda^\infty-\lambda = \lambda\left(\frac{1}{1-\frac\eta2\lambda^{-1}}-1\right) = \frac{\eta}{2-\eta\lambda^{-1}} \le \eta,
\]
where we used $\eta\le\lambda_{\min}$ once more. Since $\lambda^\infty-\lambda>0$ by \eqref{eq:ULA_biais}, this bound can be squared, and \eqref{eq:ULA_W2_maj_denominateur2} gives the result.
\end{proof}

\begin{proof}[Proof of Proposition~\ref{prop:ULA_uniform_bound}]
Combine Lemmas~\ref{lem:ULA_early_stopping_error} and~\ref{lem:ULA_discretization_error} through the triangle inequality \eqref{eq:decomp_W2}.
\end{proof}

\subsubsection{Choice of the step-size}
We now prove the first part of Theorem~\ref{thm:Langevin_rates}, namely the non-asymptotic bound for ULA.
\begin{proof}[Proof of Theorem~\ref{thm:Langevin_rates}, non-asymptotic bound for ULA]
Let $\eta_N:= \frac12\lambda_{\max}\frac{\log N}{N}$ be the ULA step-size of Theorem~\ref{thm:Langevin_rates}, and let $N\ge\frac12\kappa\log N$. This condition is exactly $\eta_N\le\lambda_{\min}$, and since $N>1$ we have $0<\eta_N\le\lambda_{\min}$. Proposition~\ref{prop:ULA_uniform_bound} therefore applies with $\eta=\eta_N$, and since $e^{-2\eta_N N/\lambda_{\max}} = 1/N$,
\[
W_2(\P_N,\P)\le\frac{\sqrt{2\Tr\Sigma}}{N} + \frac14\lambda_{\max}\sqrt{\Tr(\Sigma^{-1})}\frac{\log N}{N}.
\]
Finally,
\begin{equation}
\label{eq:UL_CS_traces}
\sqrt{\Tr\Sigma}\le\lambda_{\max}\sqrt{\Tr(\Sigma^{-1})},
\end{equation}
since the left-hand side is at most $\sqrt{d\lambda_{\max}}$ while the right-hand side is at least $\lambda_{\max}\sqrt{d/\lambda_{\max}} = \sqrt{d\lambda_{\max}}$. This gives the bound of Theorem~\ref{thm:Langevin_rates},
\[
W_2(\P_N,\P)\le\frac{\lambda_{\max}\sqrt{\Tr(\Sigma^{-1})}}{N}\left(C_1\log N + C_2\right),
\]
with $C_1=\frac14$ and $C_2=\sqrt2$.
\end{proof}

\begin{remark}[On the choice of $\eta$]
\label{rem:ULA_choice_eta}
The step-size $\eta_N$ is not the exact minimizer of the bound of Proposition~\ref{prop:ULA_uniform_bound}, but an asymptotically equivalent simplification of it. Indeed, applying Lemma~\ref{lem:optim_convexe} with
\[
A = \sqrt{2\Tr\Sigma},\quad B = \frac{2N}{\lambda_{\max}},\quad C=\frac12\sqrt{\Tr(\Sigma^{-1})},\quad D = \lambda_{\min},
\]
the infimum of that bound is attained at
\[
\eta_* = \frac{\lambda_{\max}}{2N}\log\frac{2NA}{\lambda_{\max}C}=\frac{\lambda_{\max}}{2N}\Bigg(\log N+\log\bigg(\frac{4\sqrt2}{\lambda_{\max}}\sqrt{\frac{\Tr\Sigma}{\Tr(\Sigma^{-1})}}\bigg)\Bigg),
\]
which is admissible as soon as $2NA>\lambda_{\max}C$ and $2N\ge\kappa\log\frac{2NA}{\lambda_{\max}C}$, and the minimum equals
\[
f(\eta_*)=\frac{\lambda_{\max}C}{2N}\Big(1+\log\frac{2NA}{\lambda_{\max}C}\Big).
\]
The choice $\eta_N$ amounts to keeping only the $\log N$ term in the parentheses. As a result, $f(\eta_N)$ and $f(\eta_*)$ share the leading term $\frac{\lambda_{\max}C}{2}\frac{\log N}{N}$ and differ by $O(1/N)$; in particular $f(\eta_N)/f(\eta_*)\to1$. The only loss is on the constant $C_2$, in exchange for a step-size that requires no spectral information beyond $\lambda_{\max}$, and an admissibility condition that depends on $\kappa$ only.
\end{remark}

\subsubsection{Asymptotic equivalent}
We finally prove the second part of Theorem~\ref{thm:Langevin_rates} for ULA: with the step-size $\eta=\eta_N=\frac12\lambda_{\max}\frac{\log N}{N}$,
\[
W_2(\P_N,\P) \underset{N\to\infty}{\sim} \frac18\lambda_{\max}\sqrt{\Tr(\Sigma^{-1})}\frac{\log N}N.
\]

\begin{proof}[Proof of Theorem~\ref{thm:Langevin_rates}, asymptotic equivalent for ULA]
We take $N$ large enough for $\eta_N$ to be admissible in Proposition~\ref{prop:ULA_uniform_bound}. The starting point is the closed form of Lemma~\ref{lem:ULA_closed_form},
\begin{equation}
\label{eq:ULA_forme_close_equiv_asym}
\lambda^{(N)} = \frac{\lambda}{1-\frac{\eta_N}{2}\lambda^{-1}}-\frac{\lambda}{1-\frac{\eta_N}{2}\lambda^{-1}}\left(1-\frac{\eta_N}{\lambda}\right)^{2N}.
\end{equation}
We first show that $\lambda^{(N)}\to\lambda$. On the one hand, since $\eta_N\to0$,
\begin{equation}
\label{eq:ULA_limite_linfty_equiv}
\frac{\lambda}{1-\frac{\eta_N}{2}\lambda^{-1}}\underset{N\to\infty}{\longrightarrow} \lambda.
\end{equation}
On the other hand, using $\log(1-u)\le-u$ for $u<1$ and then $\lambda\le\lambda_{\max}$,
\begin{equation}
\label{eq:ULA_borne_terme_expo}
\left(1-\frac{\eta_N}{\lambda}\right)^{2N} = \exp\left(2N\log\Big(1-\frac{\eta_N}{\lambda}\Big)\right)
\le\exp\left(-\frac{2N\eta_N}{\lambda}\right)
=N^{-\lambda_{\max}/\lambda}\le\frac1N,
\end{equation}
so the second term in \eqref{eq:ULA_forme_close_equiv_asym} tends to $0$, and with \eqref{eq:ULA_limite_linfty_equiv} we get $\lambda^{(N)}\to\lambda$. Applying Lemma~\ref{lem:W2_gaussian} directly to $W_2(\P_N,\P)$ then gives
\begin{equation}
\label{eq:ULA_W2_generic_equiv}
W_2^2(\P_N,\P) = \sum_{i=1}^d \Big(\sqrt{\li^{(N)}}-\sqrt{\li}\Big)^2
= \sum_{i=1}^d \Bigg(\frac{\li^{(N)}-\li}{\sqrt{\li^{(N)}}+\sqrt{\li}}\Bigg)^2
\sim \frac14\sum_{i=1}^d \frac{(\li^{(N)}-\li)^2}{\li}.
\end{equation}
It remains to estimate $\lambda^{(N)}-\lambda$. By \eqref{eq:ULA_forme_close_equiv_asym} and \eqref{eq:ULA_borne_terme_expo},
\begin{align*}
\lambda^{(N)}-\lambda &= \frac{\lambda}{1-\frac{\eta_N}2\lambda^{-1}}-\lambda+O\left(\frac1N\right)
= \frac{\eta_N}{2-\eta_N\lambda^{-1}}+O\left(\frac1N\right) \\
&= \frac1{2-\eta_N\lambda^{-1}}\times\frac12\lambda_{\max}\frac{\log N}N+O\left(\frac1N\right)
\sim \frac14\lambda_{\max}\frac{\log N}N,
\end{align*}
an equivalent which does not depend on $\lambda$. Plugging it into \eqref{eq:ULA_W2_generic_equiv},
\[
W_2(\P_N,\P) \underset{N\to\infty}{\sim} \frac18\lambda_{\max}\sqrt{\Tr(\Sigma^{-1})}\frac{\log N}N,
\]
which is the announced equivalent, consistent with the non-asymptotic bound.
\end{proof}

\subsection{Underdamped Langevin dynamics}
\label{app:ULD_proof}
The proof of Theorem~\ref{thm:Langevin_rates} for ULD follows the structure of the previous section. We first prove the proposition below, which bounds $W_2(\P_N,\P)$ for all $(\eta,\gamma)$ satisfying one of two sets of assumptions: a \emph{large-friction} regime, which is the one adopted in Theorem~\ref{thm:Langevin_rates}, and a \emph{small-friction} regime. Both lead to similar bounds, up to numerical constants; Remark~\ref{rem:aperiodique_naturel} explains why the first one is retained in Theorem~\ref{thm:Langevin_rates}.

\begin{proposition}[Parametric bound for ULD]
\label{prop:ULD_uniform_bound}
Let $N\ge0$. In the large-friction regime
\begin{equation}
\label{eq:ULD_hyp_aper}
\tag{LF}
\gamma\ge\sqrt{8/\lambda_{\min}}\quad\text{and}\quad\eta\le\frac1\gamma,
\end{equation}
the error of ULD satisfies
\[
W_2(\P_N,\P)\leq8\sqrt{\Tr\Sigma}\,e^{-2\eta N/(\lambda_{\max}\gamma)}+\frac67\sqrt{\Tr(\Sigma^{-1})}\,\frac\eta\gamma,
\]
and in the small-friction regime
\begin{equation}
\label{eq:ULD_hyp_osc}
\tag{SF}
\gamma\le\sqrt{2/\lambda_{\max}}\quad\text{and}\quad\eta\le\frac12\lambda_{\min}\gamma,
\end{equation}
it satisfies
\[
W_2(\P_N,\P)\leq 8\sqrt{\Tr\Sigma}\,e^{-\eta\gamma N/2}+\frac32\sqrt{\Tr(\Sigma^{-1})}\,\frac\eta\gamma.
\]
\end{proposition}

Here too, each bound is the sum of an early-stopping term and a discretization term. Up to a numerical constant, the discretization term has the same expression in both regimes. Only the early-stopping term changes, with the friction $\gamma$ appearing in the denominator or in the numerator of the exponent according to whether it is large or small.
\begin{remark}
\label{rem:aperiodique_naturel}
    The small-friction bound of Proposition~\ref{prop:ULD_uniform_bound} is strictly decreasing in $\gamma$. There is no trade-off: one simply saturates the constraint $\gamma\le\sqrt{2/\lambda_{\max}}$, which also makes the admissible range $\eta\le\frac12\lambda_{\min}\gamma$ as wide as possible. The resulting rates are the same in both regimes. We retain the large-friction regime \eqref{eq:ULD_hyp_aper} in Theorem~\ref{thm:Langevin_rates} because the roots $r_\pm$ introduced below are then real and positive, which keeps the analysis simplest, and because the bound has exactly the form of the ULA bound with $\phi:=\eta/\gamma$ in place of $\eta$; with the hyperparameters of the theorem, $\phi=\frac12\lambda_{\max}\frac{\log N}N$ is precisely the ULA step-size.
\end{remark}

\subsubsection{Closed form of the eigenvalues}
\label{app:ULD_closed_form}
As before, we start with a closed form for the eigenvalues of $\Sigma_N$ and $\Sigma_\infty$.
\begin{lemma}
\label{lem:ULD_forme_close}
Fix an iteration budget $N\ge0$. We consider two sets of assumptions:
\begin{itemize}
    \item either
    \begin{equation}
    \label{eq:ULD_regime_aper}
    \tag{LF$_0$}
    \gamma>\frac2{\sqrt{\lambda_{\min}}} \quad\textrm{and}\quad 0<\eta<\frac2\gamma,
    \end{equation}
    in which case we set
    \begin{equation}
    \label{eq:ULD_racines_reelles}
    r_\pm = 1-\frac{1}{2}\eta\big(\gamma \mp \sqrt{\gamma^2-4\lambda^{-1}}\big);
    \end{equation}
    \item or
    \begin{equation}
    \label{eq:ULD_regime_osc}
    \tag{SF$_0$}
    0<\gamma<\frac{2}{\sqrt{\lambda_{\max}}}\quad\textrm{and}\quad0<\eta<\lambda_{\min}\gamma,
    \end{equation}
    in which case we set
    \[
    r_\pm = 1-\frac{1}{2}\eta\big(\gamma\mp i\sqrt{4\lambda^{-1}-\gamma^2}\big),
    \]
    where $i$ denotes the imaginary unit.
\end{itemize}
In both cases, the eigenvalues $\lambda^{(N)}$ are exactly
\begin{equation}
\label{eq:ULD_def_l(N)}
\lambda^{(N)} = \frac{2\gamma(2-\eta\gamma+\eta^2\lambda^{-1})}{(\gamma-\eta\lambda^{-1})(4-2\eta\gamma+\eta^2\lambda^{-1})}\lambda- \frac{2\eta\gamma}{\gamma^2-4\lambda^{-1}}\left(\frac{r_+^{2N}}{1-r_+^2}-\frac{2(r_+r_-)^{N}}{1-r_+r_-}+\frac{r_-^{2N}}{1-r_-^2}\right),
\end{equation}
and they converge as $N\to\infty$ to
\begin{equation}
\label{eq:ULD_def_linfty}
\lambda^\infty = \frac{2\gamma(2-\eta\gamma+\eta^2\lambda^{-1})}{(\gamma-\eta\lambda^{-1})(4-2\eta\gamma+\eta^2\lambda^{-1})}\lambda.
\end{equation}
\end{lemma}

\begin{proof}
    Fix an index $i$ and consider the corresponding sequence of eigenvalues $(\lambda^{(k)})$. By Lemma~\ref{lem:gaussian_recursion}(iii) and Remark~\ref{rem:identification_algos},
    \begin{equation}
    \label{eq:ULD_def_lambda}
    \forall k\ge0,\quad\lambda^{(k)} = \Var(x_k),
    \end{equation}
    where $(x_k,v_k)$ is defined by $x_0 = v_0 = 0$ and
    \begin{equation*}
    \begin{cases}
    x_{k+1} = x_k+\eta v_k, \\
    v_{k+1} = (1-\eta\gamma)v_k -\eta\lambda^{-1}x_k + \sqrt{2\eta\gamma}\,\zeta_k,
    \end{cases}
    \end{equation*}
    with $(\zeta_k)$ independent standard one-dimensional Gaussians. In matrix form, $z_{k+1} = Mz_k+u\zeta_k$ with
\begin{equation}
\label{eq:ULD_def_M}
z_k = 
\begin{bmatrix}
    x_k \\
    v_k
\end{bmatrix}
,\quad M =
\begin{bmatrix}
1 & \eta \\
-\eta\lambda^{-1} & 1-\eta\gamma
\end{bmatrix}
,\quad u = 
\begin{bmatrix}
0 \\
\sqrt{2\eta\gamma}
\end{bmatrix},
\end{equation}
so that, since $z_0=0$,
\begin{equation}
\label{eq:ULD_zk}
z_k = \sum_{j=0}^{k-1} M^ju\,\zeta_{k-1-j}.
\end{equation}
Let $\beta_j$ denote the $(1,2)$ entry of $M^j$. Reading the first coordinate of \eqref{eq:ULD_zk} at $k=N$ gives
\[
x_N = \sqrt{2\eta\gamma}\sum_{j=0}^{N-1} \beta_j\zeta_{N-1-j},
\]
hence by \eqref{eq:ULD_def_lambda}
\begin{equation}
\label{eq:ULD_lambdaN_sum}
\lambda^{(N)} = 2\eta\gamma\sum_{j=0}^{N-1} \beta_j^2.
\end{equation}
It remains to compute the $\beta_j$. By the Cayley--Hamilton theorem, $M^2 - \Tr(M)M + \det(M)\I = 0$, hence ${M^{j+2} = \Tr(M)M^{j+1} - \det(M)M^j}$ for all $j$, and taking $(1,2)$ entries,
\begin{equation}
\label{eq:ULD_recurrence_beta}
\forall j\ge0,\quad \beta_{j+2} = \Tr(M)\beta_{j+1}-\det(M)\beta_j.
\end{equation}
From \eqref{eq:ULD_def_M}, $\Tr(M) = 2-\eta\gamma$ and $\det(M) = 1 - \eta\gamma +\eta^2\lambda^{-1}$, so the behavior of $(\beta_j)$ is governed by the roots of
\[
P(X) = X^2 - (2-\eta\gamma)X+1-\eta\gamma+\eta^2\lambda^{-1},
\]
whose discriminant is
\begin{equation}
\label{eq:ULD_discr_P}
\Delta = \eta^2(\gamma^2-4\lambda^{-1}).
\end{equation}
This formula shows that the two assumptions on $\gamma$ in Lemma~\ref{lem:ULD_forme_close} are exactly the cases where $\Delta$ is positive for all eigenvalues (large friction, real roots), or negative for all eigenvalues (small friction, complex roots): these are the two homogeneous regimes we study.

\emph{Stability.} By \eqref{eq:ULD_lambdaN_sum}, for $\lambda^{(N)}$ to converge the sequence $(\beta_j)$ must tend to $0$, \textit{i.e.}, both roots of $P$ must have modulus strictly less than $1$. By the Jury stability criterion \citep{jury1962simplified}, this amounts to
\[
\abs{\det(M)}<1,\quad P(1)>0,\quad P(-1)>0,
\]
that is,
\[
\abs{1-\eta\gamma+\eta^2\lambda^{-1}}<1,\quad \eta^2\lambda^{-1} > 0,\quad 4-2\eta\gamma+\eta^2\lambda^{-1}>0.
\]
The second condition holds since $\eta>0$. For the third, as $\eta^2\lambda^{-1}>0$, it suffices that $\eta\gamma\le2$: this is immediate under \eqref{eq:ULD_regime_aper}, whose assumption on $\eta$ is precisely the strict version of this inequality, and under \eqref{eq:ULD_regime_osc} it follows from
\[
P(-1) = \frac{(\eta-\gamma\lambda)^2}{\lambda} + 4-\gamma^2\lambda\ge4-\gamma^2\lambda_{\max}>0.
\]
As for the first condition, $\det(M)<1$ is equivalent to $\eta/\gamma<\lambda$, which holds since $\eta/\gamma<\lambda_{\min}\le\lambda$ under both \eqref{eq:ULD_regime_aper} (where $\eta/\gamma<2/\gamma^2<\lambda_{\min}/2$) and \eqref{eq:ULD_regime_osc}; and $\det(M)>-1$ follows from $P(1)+P(-1) = 2(1+\det(M))$, which is positive by the two conditions just established. Thus, in both cases,
\begin{equation}
\label{eq:ULD_Jury_proven}
\abs{\det(M)}<1,\quad P(1)>0,\quad P(-1)>0,
\end{equation}
and the assumptions of Lemma~\ref{lem:ULD_forme_close} guarantee stability.

\emph{Computation of $\lambda^{(N)}$.} By \eqref{eq:ULD_discr_P}, $P$ has two distinct roots, which are the $r_\pm$ of the statement: $r_\pm = 1-\frac{1}{2}\eta\big(\gamma \mp \sqrt{\gamma^2-4\lambda^{-1}}\big)$ under \eqref{eq:ULD_regime_aper} and $r_\pm = 1-\frac{1}{2}\eta\big(\gamma\mp i\sqrt{4\lambda^{-1}-\gamma^2}\big)$ under \eqref{eq:ULD_regime_osc}. They satisfy
\begin{equation}
\label{eq:ULD_sum_rpm}
r_++r_- = \Tr(M) = 2-\eta\gamma
\end{equation}
and
\begin{equation}
\label{eq:ULD_prod_rpm}
r_+r_- = \det(M) = 1-\eta\gamma+\eta^2\lambda^{-1},
\end{equation}
two relations we will use repeatedly. By \eqref{eq:ULD_recurrence_beta}, $\beta_j = Ar_+^j+Br_-^j$ for some constants $A,B$ determined by $\beta_0=0$ and $\beta_1 = \eta$, which gives
\[
\forall j\ge0,\quad \beta_j = \frac{\eta}{r_+-r_-}(r_+^j - r_-^j),
\]
and with \eqref{eq:ULD_lambdaN_sum},
\[
\lambda^{(N)} = \frac{2\eta^3\gamma}{(r_+-r_-)^2}\sum_{j=0}^{N-1} (r_+^{2j}-2(r_+r_-)^j+r_-^{2j}).
\]
The Jury criterion, satisfied under \eqref{eq:ULD_regime_aper} and \eqref{eq:ULD_regime_osc}, ensures
\begin{equation}
\label{eq:ULD_module_rpm}
\abs{r_\pm}<1,
\end{equation}
so $r_+^2$, $r_-^2$ and $r_+r_-$ all differ from $1$, and summing the three geometric series yields
\[
\lambda^{(N)} = \lambda^\infty - \frac{2\eta^3\gamma}{(r_+-r_-)^2}\left(\frac{r_+^{2N}}{1-r_+^2}-\frac{2(r_+r_-)^{N}}{1-r_+r_-}+\frac{r_-^{2N}}{1-r_-^2}\right),
\]
with
\begin{equation}
\label{eq:ULD_def_linfty_rpm}
\lambda^\infty = \frac{2\eta^3\gamma}{(r_+-r_-)^2}\left(\frac{1}{1-r_+^2}-\frac{2}{1-r_+r_-}+\frac{1}{1-r_-^2}\right).
\end{equation}
By \eqref{eq:ULD_module_rpm}, $\lambda^{(N)}$ converges to $\lambda^\infty$. It remains to show that $\lambda^\infty$ has the announced form. Putting the three fractions in \eqref{eq:ULD_def_linfty_rpm} over the common denominator $(1-r_+^2)(1-r_-^2)(1-r_+r_-)$, the numerator $C$ factorizes:
\begin{align*}
C &= (1-r_+^2)(1-r_+r_-)-2(1-r_+^2)(1-r_-^2)+(1-r_+r_-)(1-r_-^2) \\
&= r_+^2-2r_+r_-+r_-^2+r_+r_-(r_+^2-2r_+r_-+r_-^2) \\
&= (r_+-r_-)^2(1+r_+r_-),
\end{align*}
so that
\begin{equation}
\label{eq:ULD_linf_int}
\lambda^\infty = 2\eta^3\gamma\,\frac{1+r_+r_-}{(1-r_+^2)(1-r_-^2)(1-r_+r_-)}.
\end{equation}
To conclude independently of the regime, we express everything through the sum and product of the roots. On the one hand, by \eqref{eq:ULD_sum_rpm} and \eqref{eq:ULD_prod_rpm},
\begin{align*}
(1-r_+^2)(1-r_-^2) &= 1 + (r_+r_-)^2-(r_++r_-)^2+2r_+r_- \\
&= 1+(1-\eta\gamma+\eta^2\lambda^{-1})^2-(2-\eta\gamma)^2+2(1-\eta\gamma+\eta^2\lambda^{-1}) \\
&= \eta^4\lambda^{-2}+4\eta^2\lambda^{-1}-2\eta^3\gamma\lambda^{-1} = \eta^2\lambda^{-1}(4-2\eta\gamma+\eta^2\lambda^{-1}),
\end{align*}
the third line being a direct expansion of the second.
On the other hand, \eqref{eq:ULD_prod_rpm} gives $1+r_+r_- = 2-\eta\gamma+\eta^2\lambda^{-1}$ and $1-r_+r_- = \eta(\gamma-\eta\lambda^{-1})$. Plugging these into \eqref{eq:ULD_linf_int},
\[
\lambda^\infty = \frac{2\gamma(2-\eta\gamma+\eta^2\lambda^{-1})}{(\gamma-\eta\lambda^{-1})(4-2\eta\gamma+\eta^2\lambda^{-1})}\lambda,
\]
which is \eqref{eq:ULD_def_linfty}. As expected, this expression does not depend on the regime, and once again $\lambda^\infty\to\lambda$ as $\eta\to0$. Finally, \eqref{eq:ULD_def_l(N)} follows from $(r_+-r_-)^2 = (r_++r_-)^2-4r_+r_- = \eta^2(\gamma^2-4\lambda^{-1})$, by \eqref{eq:ULD_sum_rpm} and \eqref{eq:ULD_prod_rpm} once more.
\end{proof}

\subsubsection{Early-stopping error}

We now bound the early-stopping error of ULD.
\begin{lemma}
\label{lem:ULD_early_stopping_error}
Let $N\ge0$. In the large-friction regime \eqref{eq:ULD_hyp_aper}, the early-stopping error satisfies
\[
W_2(\P_N,\P_\infty)\leq8\sqrt{\Tr\Sigma}\,e^{-2\eta N/(\lambda_{\max}\gamma)},
\]
and in the small-friction regime \eqref{eq:ULD_hyp_osc},
\[
W_2(\P_N,\P_\infty)\le8\sqrt{\Tr\Sigma}\,e^{-\eta\gamma N/2}.
\]
\end{lemma}

\begin{proof}
    Once again, \eqref{eq:W2s_ULA_generic} gives
    \begin{equation}
    \label{eq:ULD_W2_maj_denominateur}
    W_2^2(\P_N,\P_\infty) = \sum_{i=1}^d \left(\frac{\li^\infty-\li^{(N)}}{\sqrt{\li^\infty}+\sqrt{\li^{(N)}}}\right)^2 \leq \sum_{i=1}^d \frac{\big(\li^\infty-\li^{(N)}\big)^2}{\li^\infty}.
    \end{equation}
    We first check that
    \begin{equation}
    \label{eq:ULD_comp_vpr}
    \lambda^\infty>\lambda.
    \end{equation}
    Starting from \eqref{eq:ULD_def_linfty},
    \begin{align}
    \label{eq:ULD_egalite_diff_vpr_discret}
    \lambda^\infty-\lambda &= \frac{\gamma(4-2\eta\gamma+2\eta^2\lambda^{-1})-(\gamma-\eta\lambda^{-1})(4-2\eta\gamma+\eta^2\lambda^{-1})}{(\gamma-\eta\lambda^{-1})(4-2\eta\gamma+\eta^2\lambda^{-1})}\lambda \notag \\
    &=\frac{\gamma\eta^2\lambda^{-1}+\eta\lambda^{-1}(4-2\eta\gamma+\eta^2\lambda^{-1})}{(\gamma-\eta\lambda^{-1})(4-2\eta\gamma+\eta^2\lambda^{-1})}\lambda \notag \\
    &= \frac{(4-\eta\gamma+\eta^2\lambda^{-1})\,\eta}{(\gamma-\eta\lambda^{-1})(4-2\eta\gamma+\eta^2\lambda^{-1})}.
    \end{align}
    Recognizing $P(-1) = 4-2\eta\gamma+\eta^2\lambda^{-1}$, which we showed to be positive, this reads
    \begin{equation}
    \label{eq:diff_vpr_biais}
    \lambda^\infty-\lambda = \frac{\eta}{\gamma-\eta\lambda^{-1}}\times\frac{P(-1)+\eta\gamma}{P(-1)}.
    \end{equation}
    The difference therefore has the sign of $\gamma-\eta\lambda^{-1}$, which is positive by the first Jury condition in \eqref{eq:ULD_Jury_proven}. This proves \eqref{eq:ULD_comp_vpr}, and \eqref{eq:ULD_W2_maj_denominateur} becomes
    \begin{equation}
    \label{eq:ULD_W2_discr_maj}
    W_2^2(\P_N,\P_\infty)\le\sum_{i=1}^d \frac{\big(\li^\infty-\li^{(N)}\big)^2}{\li}.
    \end{equation}
    We are thus led to bound $\abs{\li^\infty-\li^{(N)}}$, for which Lemma~\ref{lem:ULD_forme_close} provides
    \begin{equation}
    \label{eq:ULD_majo_abs_biais}
    \abs{\lambda^\infty-\lambda^{(N)}} = \frac{2\eta\gamma}{\abs{\gamma^2-4\lambda^{-1}}}\left|\frac{r_+^{2N}}{1-r_+^2}-\frac{2(r_+r_-)^{N}}{1-r_+r_-}+\frac{r_-^{2N}}{1-r_-^2}\right|.
    \end{equation}

\emph{Large-friction regime \eqref{eq:ULD_hyp_aper}.} By \eqref{eq:ULD_lambdaN_sum}, $\lambda^\infty$ is the limit of $2\eta\gamma\sum_{j=0}^{N-1}\beta_j^2$ as $N\to\infty$, so the difference $\lambda^\infty-\lambda^{(N)}$ is the tail
\[
\lambda^\infty-\lambda^{(N)} = 2\eta\gamma\sum_{j=N}^{\infty} \beta_j^2 \ge0.
\]
Since moreover $\gamma^2-4\lambda^{-1}>0$ under \eqref{eq:ULD_hyp_aper}, \eqref{eq:ULD_majo_abs_biais} holds without absolute values:
\begin{equation}
\label{eq:ULD_majo_abs_discr}
\abs{\lambda^{\infty}-\lambda^{(N)}} = \frac{2\eta\gamma}{\gamma^2-4\lambda^{-1}}\left(\frac{r_+^{2N}}{1-r_+^2}-\frac{2(r_+r_-)^{N}}{1-r_+r_-}+\frac{r_-^{2N}}{1-r_-^2}\right).
\end{equation}
Now $1-r_+r_->0$ by \eqref{eq:ULD_module_rpm}, and \eqref{eq:ULD_hyp_aper} makes the product of the roots positive:
\begin{equation}
\label{eq:ULD_prod_racine_pos}
    r_+r_- = 1 - \eta\gamma + \eta^2\lambda^{-1}
    > 1-\eta\gamma
    \geq 0.
\end{equation}
The middle term in \eqref{eq:ULD_majo_abs_discr} is therefore negative, and
\begin{equation}
\label{eq:ULD_majo_abs_discr_inter}
\abs{\lambda^{\infty}-\lambda^{(N)}} < \frac{2\eta\gamma}{\gamma^2-4\lambda^{-1}}\left(\frac{r_+^{2N}}{1-r_+^2}+\frac{r_-^{2N}}{1-r_-^2}\right),
\end{equation}
where every numerator and denominator is positive: \eqref{eq:ULD_hyp_aper} for $\gamma^2-4\lambda^{-1}$, and \eqref{eq:ULD_module_rpm} for $1-r_\pm^2$. We next compare $\abs{r_+}$ and $\abs{r_-}$. By \eqref{eq:ULD_prod_racine_pos} the two roots have the same sign, namely that of their sum $2-\eta\gamma>0$ by \eqref{eq:ULD_sum_rpm} and \eqref{eq:ULD_hyp_aper}, so both are positive, and their expression \eqref{eq:ULD_racines_reelles} gives
\begin{equation}
\label{eq:ULD_comp_va}
r_+=\abs{r_+}>\abs{r_-}=r_-.
\end{equation}
Hence \eqref{eq:ULD_majo_abs_discr_inter} yields
\[
\abs{\lambda^{\infty}-\lambda^{(N)}} < \frac{2\eta\gamma}{\gamma^2-4\lambda^{-1}}\left(\frac{1}{1-r_+^2}+\frac{1}{1-r_-^2}\right)r_+^{2N}.
\]
One could bound the two terms in the parenthesis separately and get a slightly sharper estimate, but this would only improve a numerical constant, so we simply use $1-r_-^2>1-r_+^2$, a consequence of \eqref{eq:ULD_comp_va}, to obtain
\[
\abs{\lambda^{\infty}-\lambda^{(N)}} < \frac{4\eta\gamma}{\gamma^2-4\lambda^{-1}}\times\frac{r_+^{2N}}{1-r_+^2}.
\]
By \eqref{eq:ULD_hyp_aper}, $\gamma^2-4\lambda^{-1}\ge\gamma^2-4\lambda_{\min}^{-1}\ge\frac12\gamma^2$, hence
\begin{equation}
\label{eq:ULD_maj_va_es}
\abs{\lambda^\infty-\lambda^{(N)}}<\frac{8\eta}\gamma\times\frac{r_+^{2N}}{1-r_+^2}.
\end{equation}
Again without optimizing constants, since $r_+>0$,
\begin{equation}
\label{eq:ULD_majo_non_opti}
\frac1{1-r_+^2} = \frac1{(1+r_+)(1-r_+)} \le \frac1{1-r_+},
\end{equation}
and by \eqref{eq:ULD_racines_reelles},
\[
\frac1{1-r_+} = \frac2{\eta\big(\gamma-\sqrt{\gamma^2-4\lambda^{-1}}\big)} = \frac{2\big(\gamma+\sqrt{\gamma^2-4\lambda^{-1}}\big)}{\eta\big(\gamma^2-(\gamma^2-4\lambda^{-1})\big)}<\frac{\gamma\lambda}{\eta}.
\]
Combining with \eqref{eq:ULD_maj_va_es} and \eqref{eq:ULD_majo_non_opti},
\begin{equation}
\label{eq:ULD_majo_inter_diff_vpr_discr}
\abs{\lambda^\infty-\lambda^{(N)}}<8\lambda\, r_+^{2N}.
\end{equation}
To make this bound usable we finally bound $r_+$. The concavity inequality $\sqrt{1-u} \leq 1 - \frac{1}{2}u$, valid for $u\in[0,1]$, gives
\begin{align*}
r_+ &= 1 - \frac{1}{2}\eta\gamma\left(1 - \sqrt{1-4\lambda^{-1}\gamma^{-2}}\right)
\leq 1 - \frac{1}{2}\eta\gamma\times\frac{2}{\lambda\gamma^2}
= 1-\frac{\eta}{\gamma\lambda}
\leq 1-\frac{\eta}{\gamma\lambda_{\max}}.
\end{align*}
Since $1-\eta/(\gamma\lambda_{\max})$ is positive under \eqref{eq:ULD_hyp_aper}, raising to the power $2N$ and using, as in Section~\ref{app:ULA_proof}, the convexity inequality
\begin{equation}
\label{eq:ULD_convexite}
\forall u \in \R,\quad 1-u\le e^{-u},
\end{equation}
we deduce from \eqref{eq:ULD_majo_inter_diff_vpr_discr}
\begin{equation}
\label{eq:ULD_maj_finale_discr_aper}
\abs{\lambda^\infty-\lambda^{(N)}} < 8\lambda\left(1-\frac\eta{\gamma\lambda_{\max}}\right)^{2N}\le 8\lambda\, e^{-2\eta N/(\gamma\lambda_{\max})}.
\end{equation}

\emph{Small-friction regime \eqref{eq:ULD_hyp_osc}.} We return to \eqref{eq:ULD_majo_abs_biais}. The roots $r_\pm$ are now complex conjugates, so
\[
\abs{r_+}^2 = \abs{r_-}^2 = r_+r_- = 1-\eta\gamma+\eta^2\lambda^{-1} =: R^2,
\]
the common value being the constant coefficient of $P$. Moreover $\frac{r_+^{2N}}{1-r_+^2}$ and $\frac{r_-^{2N}}{1-r_-^2}$ are conjugate as well, hence
\begin{align*}
\left|\frac{r_+^{2N}}{1-r_+^2}-\frac{2(r_+r_-)^{N}}{1-r_+r_-}+\frac{r_-^{2N}}{1-r_-^2}\right| &= 2\left|\text{Re}\left(\frac{r_+^{2N}}{1-r_+^2}\right)-\frac{R^{2N}}{1-R^2}\right| \\
&\leq 2\left|\frac{r_+^{2N}}{1-r_+^2}\right|+\frac{2R^{2N}}{1-R^2}
= 2R^{2N}\left(\frac{1}{\abs{1-r_+^2}}+\frac{1}{1-R^2}\right) \\
&\le \frac{4R^{2N}}{1-R^2},
\end{align*}
where the last step uses $\abs{1-r_+^2}\geq 1-\abs{r_+}^2 = 1-R^2$. Now $R^2 = 1-\eta(\gamma-\eta\lambda^{-1})$ and, by \eqref{eq:ULD_hyp_osc},
\begin{equation}
\label{eq:ULD_majo_denominateur}
\gamma-\eta\lambda^{-1}\ge\gamma-\frac12\lambda_{\min}\gamma\,\lambda_{\min}^{-1} = \frac12\gamma,
\end{equation}
so $R^2 \leq 1-\frac{1}{2}\eta\gamma$ and
\[
\left|\frac{r_+^{2N}}{1-r_+^2}-\frac{2(r_+r_-)^{N}}{1-r_+r_-}+\frac{r_-^{2N}}{1-r_-^2}\right| \leq \frac{8\left(1-\frac{1}{2}\eta\gamma\right)^N}{\eta\gamma}.
\]
Plugging this into \eqref{eq:ULD_majo_abs_biais} and using $\abs{\gamma^2-4\lambda^{-1}} = 4\lambda^{-1}-\gamma^2\ge4\lambda^{-1}-2\lambda_{\max}^{-1} \ge 2\lambda^{-1}$,
\[
\abs{\lambda^\infty-\lambda^{(N)}} \leq \frac{16}{\abs{\gamma^2-4\lambda^{-1}}}\left(1-\frac{1}{2}\eta\gamma\right)^N \le 8\lambda\left(1-\frac12\eta\gamma\right)^N.
\]
Since $1-\frac12\eta\gamma\ge0$ under \eqref{eq:ULD_hyp_osc}, the convexity inequality \eqref{eq:ULD_convexite} finally gives
\begin{equation}
\label{eq:ULD_maj_finale_discr_osc}
    \abs{\lambda^\infty-\lambda^{(N)}} \le 8\lambda\, e^{-\eta\gamma N/2}.
\end{equation}

\emph{Conclusion.} Plugging \eqref{eq:ULD_maj_finale_discr_aper}, respectively \eqref{eq:ULD_maj_finale_discr_osc}, into \eqref{eq:ULD_W2_discr_maj} gives
\begin{align*}
W_2(\P_N,\P_\infty) &\le 8\sqrt{\Tr\Sigma}\,e^{-2\eta N/(\lambda_{\max}\gamma)} &&\text{under \eqref{eq:ULD_hyp_aper}}, \\
W_2(\P_N,\P_\infty) &\le 8\sqrt{\Tr\Sigma}\,e^{-\eta\gamma N/2} &&\text{under \eqref{eq:ULD_hyp_osc}}.
\end{align*}
\end{proof}

\begin{remark}
\label{rem:ULD_stricter}
    The large- and small-friction assumptions \eqref{eq:ULD_hyp_aper} and \eqref{eq:ULD_hyp_osc} are stricter than conditions \eqref{eq:ULD_regime_aper} and \eqref{eq:ULD_regime_osc} of Lemma~\ref{lem:ULD_forme_close}, but only through numerical factors: their nature is the same. This stricter framework, which defines the admissible hyperparameters for ULD, keeps the denominator $\gamma^2-4\lambda^{-1}$ bounded away from zero, and affects the early-stopping and discretization bounds by a constant factor only.
\end{remark}

\subsubsection{Discretization error}
We now bound the discretization error.
\begin{lemma}
\label{lem:ULD_discretization_error}
    Under \eqref{eq:ULD_hyp_aper} or \eqref{eq:ULD_hyp_osc}, the discretization error satisfies
    \[
    W_2(\P_\infty,\P)\le C\sqrt{\Tr(\Sigma^{-1})}\,\frac\eta\gamma,
    \]
    with $C = \frac67$ in the large-friction regime and $C = \frac32$ in the small-friction regime.
\end{lemma}

\begin{proof}
    By \eqref{eq:W2s_ULA_generic} and the inequality \eqref{eq:ULD_comp_vpr} established in the previous section,
    \begin{equation}
    \label{eq:ULD_W2_discr_maj_generique}
    W_2^2(\P_\infty,\P) = \sum_{i=1}^d \left(\frac{\li^\infty-\li}{\sqrt{\li^\infty}+\sqrt{\li}}\right)^2 \le \frac14\sum_{i=1}^d \frac{(\li^\infty-\li)^2}{\li}.
    \end{equation}
    The lemma will again follow from a bound on $\abs{\lambda^\infty-\lambda} = \lambda^\infty-\lambda$, for which we return to \eqref{eq:diff_vpr_biais}:
    \begin{equation}
    \label{eq:ULD_diff_vpr_discr}
    \lambda^\infty-\lambda = \frac{\eta}{\gamma-\eta\lambda^{-1}}\times\frac{P(-1)+\eta\gamma}{P(-1)} = \frac{\eta}{\gamma-\eta\lambda^{-1}}\left(1+\frac{\eta\gamma}{P(-1)}\right),
    \end{equation}
    where $P(-1)$ is the quantity introduced with the Jury criterion. The assumptions of Proposition~\ref{prop:ULD_uniform_bound}, stronger than those of Lemma~\ref{lem:ULD_forme_close} where this criterion appears, keep $P(-1)$ bounded away from zero. Indeed, in both regimes $\eta\gamma\le1$: this is immediate under \eqref{eq:ULD_hyp_aper}, and under \eqref{eq:ULD_hyp_osc},
    \[
    \eta\gamma\le\frac12\lambda_{\min}\gamma^2\le\frac1\kappa\le1.
    \]
    Consequently $P(-1) = 4 - 2\eta\gamma+\eta^2\lambda^{-1} > 4-2\eta\gamma \ge 2$ in both regimes, and \eqref{eq:ULD_diff_vpr_discr} gives
    \[
    \lambda^\infty-\lambda < \frac32\times\frac\eta{\gamma-\eta\lambda^{-1}}.
    \]
    It remains to bound $\gamma-\eta\lambda^{-1}$ from below. Under \eqref{eq:ULD_hyp_aper},
    \[
    \gamma-\eta\lambda^{-1} = \gamma\left(1-\frac\eta{\gamma\lambda}\right) \ge \gamma\left(1-\frac1{\lambda\gamma^2}\right) \ge \gamma\left(1-\frac{\lambda_{\min}}{8\lambda}\right)\ge\frac78\gamma,
    \]
    hence
    \begin{equation}
    \label{eq:ULD_borne_aper_discr_vpr}
    \lambda^\infty-\lambda<\frac{12\eta}{7\gamma}.
    \end{equation}
    Under \eqref{eq:ULD_hyp_osc}, \eqref{eq:ULD_majo_denominateur} gives $\gamma-\eta\lambda^{-1}\ge\frac12\gamma$, hence
    \begin{equation}
    \label{eq:ULD_borne_osc_discr_vpr}
    \lambda^\infty-\lambda<\frac{3\eta}\gamma.
    \end{equation}
    Plugging \eqref{eq:ULD_borne_aper_discr_vpr} and \eqref{eq:ULD_borne_osc_discr_vpr} into \eqref{eq:ULD_W2_discr_maj_generique} completes the proof.
\end{proof}

\begin{proof}[Proof of Proposition~\ref{prop:ULD_uniform_bound}]
Combine Lemmas~\ref{lem:ULD_early_stopping_error} and~\ref{lem:ULD_discretization_error} through the triangle inequality \eqref{eq:decomp_W2}.
\end{proof}

\subsubsection{Choice of hyperparameters}

We can now prove the first part of Theorem~\ref{thm:Langevin_rates} for ULD.
\begin{proof}[Proof of Theorem~\ref{thm:Langevin_rates}, non-asymptotic bound for ULD]
Let $N\ge2$ satisfy $N\ge4\kappa\log N$, and recall the ULD hyperparameters of Theorem~\ref{thm:Langevin_rates}: $\eta_N = \sqrt{2\lambda_{\max}\kappa}\,\frac{\log N}{N}$ and $\gamma = \sqrt{8/\lambda_{\min}}$. They satisfy the large-friction assumption \eqref{eq:ULD_hyp_aper}: the only point to check is $\eta_N\gamma\le1$, and indeed
\[
\eta_N\gamma = 4\kappa\frac{\log N}{N}\le1
\]
by assumption. The first bound of Proposition~\ref{prop:ULD_uniform_bound} therefore applies. Since $\frac{\eta_N}{\gamma} = \frac12\lambda_{\max}\frac{\log N}N$, the exponential equals $e^{-\log N} = 1/N$, and
\[
W_2(\P_N,\P)\le\frac{8\sqrt{\Tr\Sigma}}{N}+\frac37\lambda_{\max}\sqrt{\Tr(\Sigma^{-1})}\frac{\log N}N.
\]
Using \eqref{eq:UL_CS_traces} once more gives the announced bound,
\[
W_2(\P_N,\P)\le\frac{\lambda_{\max}\sqrt{\Tr(\Sigma^{-1})}}N\Big(\frac37\log N+8\Big).
\]
\end{proof}

The same argument in the small-friction regime \eqref{eq:ULD_hyp_osc} gives a bound of the same order.
\begin{corollary}
\label{cor:ULD_underdamped}
    Let $N\ge2\kappa\log N$. With $\eta = \sqrt{2\lambda_{\max}}\,\frac{\log N}N$ and $\gamma = \sqrt{2/\lambda_{\max}}$,
    \[
    W_2(\P_N,\P)\le \frac{\lambda_{\max}\sqrt{\Tr(\Sigma^{-1})}}N\Big(\frac32\log N + 8\Big).
    \]
\end{corollary}
\begin{proof}
With these hyperparameters, $\eta\gamma = 2\frac{\log N}{N}$ and $\frac\eta\gamma = \lambda_{\max}\frac{\log N}{N}$, and the condition $N\ge2\kappa\log N$ is exactly $\eta\le\frac12\lambda_{\min}\gamma$, so \eqref{eq:ULD_hyp_osc} holds. The second bound of Proposition~\ref{prop:ULD_uniform_bound} gives 
\[
W_2(\P_N,\P)\le 8\sqrt{\Tr\Sigma}/N+\frac32\lambda_{\max}\sqrt{\Tr(\Sigma^{-1})}\frac{\log N}N,
\]
and \eqref{eq:UL_CS_traces} concludes.
\end{proof}

\begin{remark}
    To determine the optimal parameters $\eta_*,\gamma_*$, one applies Lemma~\ref{lem:optim_convexe} to the variable $\phi :=\eta/\gamma$. As for ULA, the exact minimizers are slightly different, with spectral coefficients inside the logarithms, but lead to the same bound up to numerical constants. The hyperparameters we retain are simpler: they require no knowledge of the spectrum of $\Sigma$ beyond its extreme eigenvalues. The same principles guided the hyperparameters of Corollary~\ref{cor:ULD_underdamped}, where Lemma~\ref{lem:optim_convexe} applies thanks to Remark~\ref{rem:aperiodique_naturel}, which removes the variable $\gamma$. In what follows we nevertheless focus on the regime of Theorem~\ref{thm:Langevin_rates}. As for ULA, the asymptotic equivalent of the next section shows that, for the chosen hyperparameters, the actual error is also of order $\lambda_{\max}\sqrt{\Tr(\Sigma^{-1})}\log N/N$.
\end{remark}

\subsubsection{Asymptotic equivalent}
To complete the proof of Theorem~\ref{thm:Langevin_rates}, it remains to establish the asymptotic equivalent
\[
W_2(\P_N,\P)\underset{N\to\infty}{\sim} \frac14\lambda_{\max}\sqrt{\Tr(\Sigma^{-1})}\frac{\log N}N
\]
for $\eta = \eta_N = \sqrt{2\lambda_{\max}\kappa}\,\frac{\log N}{N}$ and $\gamma = \sqrt{8/\lambda_{\min}}$.

\begin{proof}[Proof of Theorem~\ref{thm:Langevin_rates}, asymptotic equivalent for ULD]
Take $N$ large enough for the hyperparameters to satisfy \eqref{eq:ULD_hyp_aper}, so that Lemma~\ref{lem:ULD_forme_close} and the bound \eqref{eq:ULD_maj_finale_discr_aper} apply. Recall from \eqref{eq:ULD_def_linfty} that
\[
\lambda^\infty = \frac{2\gamma(2-\eta_N\gamma+\eta_N^2\lambda^{-1})}{(\gamma-\eta_N\lambda^{-1})(4-2\eta_N\gamma+\eta_N^2\lambda^{-1})}\lambda.
\]
We first show that $\lambda^{(N)}\to\lambda$. On the one hand, $\eta_N\to0$, so
\begin{equation}
\label{eq:ULD_limite_linfty}
\lambda^\infty \longrightarrow \frac{4\gamma}{4\gamma}\lambda = \lambda.
\end{equation}
On the other hand, \eqref{eq:ULD_maj_finale_discr_aper} with the values of the hyperparameters gives
\[
\abs{\lambda^\infty-\lambda^{(N)}} \le 8\lambda\, e^{-2\eta_N N/(\gamma\lambda_{\max})} = \frac{8\lambda}N,
\]
and with \eqref{eq:ULD_limite_linfty}, $\lambda^{(N)}\to\lambda$ indeed. Applying Lemma~\ref{lem:W2_gaussian} to $\P_N$ and $\P$ then gives
\begin{equation}
\label{eq:ULD_equiv_W2_vpr}
W_2^2(\P_N,\P) = \sum_{i=1}^d \Big(\sqrt{\li^{(N)}}-\sqrt{\li}\Big)^2
= \sum_{i=1}^d \Bigg(\frac{\li^{(N)}-\li}{\sqrt{\li^{(N)}}+\sqrt{\li}}\Bigg)^2
\sim \frac14\sum_{i=1}^d \frac{(\li^{(N)}-\li)^2}{\li}.
\end{equation}
It remains to estimate $\lambda^{(N)}-\lambda$. By the bound above,
\begin{equation}
\label{eq:ULD_diff_vpr_equiv_asym}
\lambda^{(N)}-\lambda = (\lambda^{(N)}-\lambda^\infty)+(\lambda^\infty-\lambda) = \lambda^\infty-\lambda+O\left(\frac1N\right),
\end{equation}
and by \eqref{eq:ULD_egalite_diff_vpr_discret},
\[
\lambda^\infty-\lambda = \frac{(4-\eta_N\gamma+\eta_N^2\lambda^{-1})\,\eta_N}{(\gamma-\eta_N\lambda^{-1})(4-2\eta_N\gamma+\eta_N^2\lambda^{-1})}\underset{N\to\infty}{\sim}\frac{4\eta_N}{4\gamma} = \frac{\eta_N}{\gamma} = \frac12\lambda_{\max}\frac{\log N}N.
\]
Together with \eqref{eq:ULD_diff_vpr_equiv_asym},
\[
\lambda^{(N)}-\lambda = \frac12\lambda_{\max}\frac{\log N}N+O\left(\frac1N\right),
\]
an equivalent which, as for ULA, does not depend on $\lambda$. The equivalent \eqref{eq:ULD_equiv_W2_vpr} then yields
\[
W_2(\P_N,\P)\underset{N\to\infty}{\sim}\frac14\lambda_{\max}\sqrt{\Tr(\Sigma^{-1})}\frac{\log N}N.
\]
\end{proof}

\section{Proof of Theorem~\ref{thm:diffusion_spectral_bound} and Proposition~\ref{prop:diffusion_coarse_bound} (diffusion sampler)}
\label{app:thm2_proof}
\subsection{Overview}

We now turn to the proof of the following theorem, stated in Section~\ref{results_diffusion}.

\DiffusionRates*

Proposition~\ref{prop:diffusion_coarse_bound}, which differs from Theorem~\ref{thm:diffusion_spectral_bound} only through the choice of hyperparameters, is restated and proved in Section~\ref{app:diffusion_hyperparameters}.

\paragraph{Error decomposition.}
The error of the diffusion algorithm of Section~\ref{problem_setup} combines three sources. The first is the initialization error: the process is started from $Y_0\sim\Nml(0,T\I)$ as a proxy for the true marginal $\bb{Q}_T = \Nml(0,\Sigma+T\I)$, which is unknown \textit{a priori}. It is measured by the Wasserstein distance $W_2(\P_N,\tilde{\P}_N)$ between the law $\P_N$ actually sampled and the law $\tilde{\P}_N$ that would have been sampled from the exact initialization $Y_0\sim\Nml(0,\Sigma+T\I)$. The second is the discretization error incurred when simulating the reverse SDE~\eqref{eq:reverse_SDE}: it is $W_2(\tilde{\P}_N,\bb{Q}_\delta)$, the distance, for an exact initialization, between the law produced by the discrete scheme and the law $\bb{Q}_\delta = \Nml(0,\Sigma+\delta\I)$ reached by the continuous process. The third is the early-stopping error $W_2(\bb{Q}_{\delta},\P)$, due to stopping the continuous process at the threshold $\delta>0$. The triangle inequality formalizes this decomposition:
\begin{equation}
\label{eq:diffusion_decomp_W2}
W_2(\P_N,\P)\le W_2(\P_N,\tilde{\P}_N) + W_2(\tilde{\P}_N,\bb{Q}_\delta) + W_2(\bb{Q}_\delta,\P).
\end{equation}

\paragraph{Reduction to eigenvalues.}
All the laws in \eqref{eq:diffusion_decomp_W2} are centered Gaussians with covariance in $\mathcal{A}(\Sigma)$: this is immediate for $\P$ and $\bb{Q}_\delta$, and follows from Lemma~\ref{lem:gaussian_recursion} for $\P_N$ and $\tilde{\P}_N$. Lemma~\ref{lem:W2_gaussian} therefore applies and, denoting by $\li^{(N)}$ and $\tilde\lambda_i^{(N)}$ the eigenvalues of the covariances of $\P_N$ and $\tilde\P_N$, we record for later use
\begin{align}
W_2^2(\P_N,\tilde{\P}_N) &= \sum_{i=1}^d\Big(\sqrt{\lambda_i^{(N)}} - \sqrt{\tilde{\lambda}_i^{(N)}}\Big)^2, \label{eq:diffusion_W2s_generic_1}\\
W_2^2(\tilde{\P}_N,\bb{Q}_\delta) &= \sum_{i=1}^d\Big(\sqrt{\tilde{\lambda}_i^{(N)}} - \sqrt{\li+\delta}\Big)^2, \label{eq:diffusion_W2s_generic_2} \\
W_2^2(\bb{Q}_\delta,\P) &= \sum_{i=1}^d\Big(\sqrt{\li+\delta} - \sqrt{\li}\Big)^2. \label{eq:diffusion_W2s_generic_3}
\end{align}

\paragraph{Outline.}
The proof of Theorem~\ref{thm:diffusion_spectral_bound} follows the same steps as that of Theorem~\ref{thm:Langevin_rates}.
\begin{itemize}
    \item[(i)--(ii)] \emph{Parametric bound} (Section~\ref{app:diffusion_param_bound}). Closed form of the eigenvalues $\li^{(N)}$ and $\tilde\li^{(N)}$ from the scalar recursions of Lemma~\ref{lem:gaussian_recursion}, then separate bounds on the three terms of \eqref{eq:diffusion_decomp_W2}, and summation. This gives a bound on $W_2(\P_N,\P)$ in which $T$ and $\delta$ remain free (Proposition~\ref{prop:diffusion_param_bound}).
    \item[(iii)] \emph{Choice of hyperparameters} (Section~\ref{app:diffusion_hyperparameters}). Plug the hyperparameters of Theorem~\ref{thm:diffusion_spectral_bound} and Proposition~\ref{prop:diffusion_coarse_bound} into the bound of Proposition~\ref{prop:diffusion_param_bound}.
    \item[(iv)] \emph{Asymptotic equivalent} (Section~\ref{app:diffusion_equivalent_asymptotique}). Exact equivalent of $W_2(\P_N,\P)$ as $N\to\infty$ for any geometric grid with $T\propto\sqrt N$ and $\delta\propto N^{-2}$, from the closed forms of step~(i). This equivalent does not depend on the proportionality constants, which shows that the spectral rate is a robust feature of the scheme.
\end{itemize}

\subsection{Parametric bound for the diffusion sampler}
\label{app:diffusion_param_bound}
The goal of this section is to prove the following proposition, which bounds $W_2(\P_N,\P)$ for all values of the hyperparameters. As before, we freely drop the eigenvalue index $i$ in the proofs.

\begin{proposition}[Parametric bound for the diffusion sampler]
\label{prop:diffusion_param_bound}
    For every $N\ge 1$ and every $T>\delta>0$, with the geometric discretization $\tau_k = T(\delta/T)^{k/N}$, the error of the diffusion algorithm satisfies
    \[
    W_2(\P_N, \P) \leq \frac{\sqrt{\Tr((\Sigma+\delta\I)^3\Sigma^2)}}{T^2} + \frac{\sqrt{\Tr\Sigma+4d\delta}}{4N}\Big(\frac{T}{\delta}\Big)^{1/N}\log\frac{T}{\delta} + \sqrt{d\delta}.
    \]
\end{proposition}
The three terms bound the initialization, discretization and early-stopping errors respectively.

\subsubsection{Closed form of the eigenvalues}
\label{app:diffusion_closed_form}
We first establish a closed form for the eigenvalues $\lambda^{(N)}$ and $\tilde{\lambda}^{(N)}$, valid for any discretization grid.
\begin{lemma}
\label{lem:diffusion_closed_form}
    Fix $N\ge1$ and consider the following generalized diffusion algorithm:
    \begin{equation*}
    \begin{cases}
    Y_0 \sim \Nml(0,M), \\
    \forall k=1,\dots,N\,,\, Y_{k}= Y_{k-1}+\eta_ks_{\tau_{k-1}}(Y_{k-1})+\sqrt{\eta_k}\xi_k,
    \end{cases}
    \end{equation*}
    for some $M\in\mathcal{A}(\Sigma)$, with eigenvalues $\mu_i:=[M]_i$.
    Then the iterates $(Y_k)$ are centered Gaussian vectors whose covariance matrices $(M_k)$ belong to $\mathcal{A}(\Sigma)$, and the eigenvalues $\mu_i^{(N)}:=[M_N]_i$ are given by
    \[
    \forall i,\quad \mu_i^{(N)} = \left(\frac{\lambda_i+\delta}{\lambda_i + T}\right)^2\mu_i+(\lambda_i+\delta)^2\sum_{j=1}^N\frac{\eta_j}{(\lambda_i+\tau_j)^2}.
    \]
\end{lemma}

\begin{proof}
    This generalization of \eqref{eq:diffusion_algorithm} falls within the scope of Lemma~\ref{lem:gaussian_recursion}: the identification is that of Remark~\ref{rem:identification_algos}, the only change being the initial covariance $M$, which lies in $\mathcal{A}(\Sigma)$ by assumption. Point (ii) of the lemma gives the recursion
    \[
    \forall k\ge0,\quad\mu^{(k+1)} = \left(1-\frac{\eta_{k+1}}{\lambda+\tau_k}\right)^2\mu^{(k)} + \eta_{k+1}.
    \]
    Since $\eta_{k+1} = \tau_k-\tau_{k+1}$, the contraction factor telescopes, leading to the simplification announced in Section~\ref{problem_setup}:
    \[
        \mu^{(k+1)} = \left(\frac{\lambda + \tau_k-\eta_{k+1}}{\lambda+ \tau_k}\right)^2\mu^{(k)} + \eta_{k+1}
        = \left(\frac{\lambda+\tau_{k+1}}{\lambda+\tau_k}\right)^2\mu^{(k)} + \eta_{k+1}.
    \]
    Dividing by $(\lambda+\tau_{k+1})^2$ and setting
    \begin{equation*}
    \forall j\ge0,\quad\nu_j := \frac{\mu^{(j)}}{(\lambda + \tau_j)^2},
    \end{equation*}
    we are left with the much simpler recursion
    \begin{equation*}
    \forall j\ge0,\quad
    \nu_{j+1} = \nu_j + \frac{\eta_{j+1}}{(\lambda+\tau_{j+1})^2},
    \quad \nu_0 = \frac{\mu}{(\lambda+T)^2},
    \end{equation*}
    whence
    \[
    \nu_k = \frac{\mu}{(\lambda+T)^2}+ \sum_{j=1}^k \frac{\eta_j}{(\lambda+\tau_j)^2}
    \quad\text{and}\quad
    \mu^{(k)} = \left(\frac{\lambda+\tau_k}{\lambda+T}\right)^2\mu + (\lambda+\tau_k)^2\sum_{j=1}^k\frac{\eta_j}{(\lambda+\tau_j)^2}.
    \]
    Taking $k=N$, with $\tau_N=\delta$, gives the announced formula.
\end{proof}
\begin{remark}
\label{rem:diffusion_forme_close_vpr}
    Applying this lemma with $M = T\I$ and with $M = \Sigma+T\I$, both in $\mathcal{A}(\Sigma)$, gives for every $N\ge1$
    \begin{align*}
        \lambda^{(N)} &= \left(\frac{\lambda+\delta}{\lambda+T}\right)^2T+ (\lambda+\delta)^2\sum_{j=1}^N\frac{\eta_j}{(\lambda+\tau_j)^2}, \\
        \tilde{\lambda}^{(N)} &= \frac{(\lambda+\delta)^2}{\lambda+T}+ (\lambda+\delta)^2\sum_{j=1}^N\frac{\eta_j}{(\lambda+\tau_j)^2}.
    \end{align*}
\end{remark}

\subsubsection{Initialization error}
\begin{lemma}
\label{lem:diffusion_init_error}
    For every $N\ge1$ and every $T>\delta>0$, the initialization error is bounded uniformly in $N$:
    \[
    W_2(\P_N,\tilde\P_N) \le \frac{\sqrt{\Tr((\Sigma+\delta\I)^3\Sigma^2)}}{T^2}.
    \]
\end{lemma}

\begin{proof}
Fix $N\ge1$ and $T>\delta>0$ throughout. By \eqref{eq:diffusion_W2s_generic_1},
\begin{equation}
\label{eq:diffusion_init_error_W2_generic}
W_2^2(\P_N,\tilde{\P}_N) = \sum_{i=1}^d\Big(\sqrt{\lambda_i^{(N)}} - \sqrt{\tilde{\lambda}_i^{(N)}}\Big)^2.
\end{equation}
Starting from the expression of $\tilde\lambda^{(N)}$ in Remark~\ref{rem:diffusion_forme_close_vpr}, and using $(\lambda+\delta)^2\cdot\frac1{\lambda+\delta} = \lambda+\delta$,
\begin{align}
    \tilde\lambda^{(N)} &= (\lambda+\delta)^2\Bigg(\sum_{j=1}^N\frac{\eta_j}{(\lambda+\tau_j)^2}+\frac1{\lambda+T}\Bigg) \notag \\
    &= \lambda+\delta + (\lambda+\delta)^2\Bigg(\sum_{j=1}^N\frac{\eta_j}{(\lambda+\tau_j)^2}-\bigg(\frac1{\lambda+\delta}-\frac1{\lambda+T}\bigg)\Bigg)
    = \lambda+\delta + (\lambda+\delta)^2\Delta, \label{eq:diffusion_initialization_error_comp_vpr}
\end{align}
where
\[
\Delta := \sum_{j=1}^N\frac{\eta_j}{(\lambda+\tau_j)^2}-\bigg(\frac1{\lambda+\delta}-\frac1{\lambda+T}\bigg)
= \sum_{j=1}^N\frac{\eta_j}{(\lambda+\tau_j)^2} - \int_\delta^T\frac{\dd\tau}{(\lambda+\tau)^2}.
\]
$\Delta$ is the quadrature error made when approximating $\int_\delta^T (\lambda+\tau)^{-2}\dd\tau$ by the Riemann sum $\sum_{j=1}^N\eta_j(\lambda+\tau_j)^{-2}$, in which the decreasing integrand is evaluated at the left endpoint of each interval $[\tau_{j},\tau_{j-1}]$; it is therefore positive. Explicitly, since $\tau_0=T$ and $\tau_N=\delta$,
\begin{align}
    \Delta &= \sum_{j=1}^N \frac{\eta_j}{(\lambda+\tau_j)^2} - \sum_{j=1}^N \left(\frac{1}{\lambda+\tau_j}-\frac{1}{\lambda+\tau_{j-1}}\right) \notag \\
    &= \sum_{j=1}^N \left(\frac{\eta_j}{(\lambda+\tau_j)^2} - \frac{\eta_j}{(\lambda+\tau_j)(\lambda+\tau_{j-1})}\right) \label{eq:diffusion_init_error_delta_int} \\
    &= \sum_{j=1}^N \frac{\eta_j^2}{(\lambda+\tau_j)^2(\lambda+\tau_{j-1})} > 0, \notag
\end{align}
where we used $\tau_{j-1}-\tau_j = \eta_j$ twice. With \eqref{eq:diffusion_initialization_error_comp_vpr}, this gives
\begin{equation}
\label{eq:diffusion_initialization_error_comp_vpr_bis}
\tilde\lambda^{(N)}>\lambda+\delta.
\end{equation}
Remark~\ref{rem:diffusion_forme_close_vpr} also gives the difference of the two eigenvalues,
\begin{equation}
\label{eq:diffusion_init_error_diff_vpr}
\tilde\lambda^{(N)}-\lambda^{(N)} = \left(\frac{\lambda+\delta}{\lambda+T}\right)^2\lambda,
\end{equation}
so that, by \eqref{eq:diffusion_initialization_error_comp_vpr_bis},
\[
\sqrt{\tilde\lambda^{(N)}}-\sqrt{\lambda^{(N)}} = \frac{\tilde\lambda^{(N)}-\lambda^{(N)}}{\sqrt{\tilde\lambda^{(N)}}+\sqrt{\lambda^{(N)}}}
\le \frac{\tilde\lambda^{(N)}-\lambda^{(N)}}{\sqrt{\lambda+\delta}}
=\frac{(\lambda+\delta)^{3/2}\lambda}{(\lambda+T)^2}
\le \frac{(\lambda+\delta)^{3/2}\lambda}{T^2}.
\]
Squaring, summing over $i$ and using \eqref{eq:diffusion_init_error_W2_generic} gives the bound of the lemma.
\end{proof}

\subsubsection{Discretization error}
\begin{lemma}
\label{lem:diffusion_discretization_error}
    For every $N\ge1$ and every $T>\delta>0$, with the geometric discretization $\tau_k = T(\delta/T)^{k/N}$, the discretization error satisfies
    \[
    W_2(\tilde\P_N,\bb{Q}_\delta) \le \frac{\sqrt{\Tr\Sigma+4d\delta}}{4N}\Big(\frac{T}{\delta}\Big)^{1/N}\log\frac{T}{\delta}.
    \]
\end{lemma}

\begin{proof}
    Fix $N\ge1$ and $T>\delta>0$. By \eqref{eq:diffusion_W2s_generic_2} and \eqref{eq:diffusion_initialization_error_comp_vpr_bis},
    \begin{align*}
    W_2^2(\tilde\P_N,\bb{Q}_\delta) &= \sum_{i=1}^d\Big(\sqrt{\tilde{\lambda}_i^{(N)}} - \sqrt{\li+\delta}\Big)^2
    = \sum_{i=1}^d \Bigg(\frac{\tilde\lambda_i^{(N)}-(\lambda_i+\delta)}{\sqrt{\tilde\lambda_i^{(N)}}+\sqrt{\lambda_i+\delta}}\Bigg)^2 \\
    &\le \frac14 \sum_{i=1}^d \frac{\big(\tilde\lambda_i^{(N)}-(\lambda_i+\delta)\big)^2}{\lambda_i+\delta},
    \end{align*}
    and with \eqref{eq:diffusion_initialization_error_comp_vpr},
    \begin{equation}
    \label{eq:diffusion_discr_error_W2_generique}
    W_2^2(\tilde\P_N,\bb{Q}_\delta) \le \frac14\sum_{i=1}^d (\li+\delta)^3\Delta_i^2.
    \end{equation}
    It thus remains to bound $\Delta$ for the geometric grid $\tau_k = T(\delta/T)^{k/N}$. Let $q:=(\delta/T)^{1/N}$ denote its ratio, so that $\tau_k = q\tau_{k-1}$ and
    \[
    \forall k\ge1,\quad\eta_k = (q^{-1}-1)\tau_k.
    \]
    Returning to \eqref{eq:diffusion_init_error_delta_int},
    \begin{align}
    \Delta &= (q^{-1}-1)\sum_{j=1}^N\frac{\tau_j}{\lambda+\tau_j}\left(\frac1{\lambda+\tau_j}-\frac1{\lambda+\tau_{j-1}}\right) \label{eq:diffusion_discr_error_delta_inter_0} \\
    &= (q^{-1}-1)\sum_{j=1}^N f(\tau_j)\big(G(\tau_j)-G(\tau_{j-1})\big), \label{eq:diffusion_discr_error_delta_inter}
    \end{align}
    where
    \[
    f : \tau\longmapsto\frac{\tau}{\lambda+\tau},\quad G : \tau\longmapsto\frac1{\lambda+\tau}.
    \]
    Since $\tau_j<\tau_{j-1}$ and $f$ is increasing,
    \[
    f(\tau_j)\big(G(\tau_j)-G(\tau_{j-1})\big) = f(\tau_j)\int_{\tau_{j}}^{\tau_{j-1}}\frac{\dd\tau}{(\lambda+\tau)^2}\le \int_{\tau_j}^{\tau_{j-1}}\frac{f(\tau)\dd\tau}{(\lambda+\tau)^2}=\int_{\tau_{j}}^{\tau_{j-1}}\frac{\tau\dd\tau}{(\lambda+\tau)^3},
    \]
    and summing over $j$ in \eqref{eq:diffusion_discr_error_delta_inter},
    \begin{equation}
    \label{eq:diffusion_discr_error_majo_delta_integrale}
    \Delta\le(q^{-1}-1)\int_\delta^T\frac{\tau\dd\tau}{(\lambda+\tau)^3}.
    \end{equation}
    One checks that $\tau\mapsto\frac\lambda{2(\lambda+\tau)^2}-\frac1{\lambda+\tau}$ is an antiderivative of the integrand, hence
    \begin{align*}
    \Delta\le(q^{-1}-1)\left[\frac\lambda{2(\lambda+\tau)^2}-\frac1{\lambda+\tau}\right]_\delta^T &= (q^{-1}-1)\left(\frac{\lambda+2\delta}{2(\lambda+\delta)^2} - \frac{\lambda+2T}{2(\lambda+T)^2}\right) \\
    &\le(q^{-1}-1)\frac{\lambda+2\delta}{2(\lambda+\delta)^2},
    \end{align*}
    and therefore
    \[
    (\lambda+\delta)^3\Delta^2\le\frac14(q^{-1}-1)^2\frac{(\lambda+2\delta)^2}{\lambda+\delta}\le\frac14(q^{-1}-1)^2(\lambda+4\delta),
    \]
    where the last inequality is $(\lambda+2\delta)^2 \le (\lambda+2\delta)^2+\lambda\delta = (\lambda+\delta)(\lambda+4\delta)$. Plugging this into \eqref{eq:diffusion_discr_error_W2_generique},
    \begin{equation}
    \label{eq:diffusion_discr_error_W2_final}
    W_2(\tilde\P_N,\bb{Q}_\delta) \le \frac14(q^{-1}-1)\sqrt{\Tr\Sigma+4d\delta}.
    \end{equation}
    Finally, the inequality $e^u-1\le ue^u$, valid for all $u\in\R$, applied to $u = \frac1N\log\frac T\delta$ gives
    \[
    q^{-1}-1 = \Big(\frac T\delta\Big)^{1/N}-1\le\frac1N\Big(\frac{T}\delta\Big)^{1/N}\log\frac T\delta,
    \]
    which together with \eqref{eq:diffusion_discr_error_W2_final} concludes the proof.
\end{proof}

\subsubsection{Early-stopping error}
\begin{lemma}
\label{lem:diffusion_es_error}
    For every $\delta>0$, the early-stopping error satisfies
    \[
    W_2(\bb{Q}_\delta,\P) \le \sqrt{d\delta}.
    \]
\end{lemma}

\begin{proof}
    By \eqref{eq:diffusion_W2s_generic_3}, it suffices to note that for every $\lambda>0$,
    \[
    (\sqrt{\lambda+\delta}-\sqrt{\lambda})^2 = 2\lambda + \delta -2\sqrt{\lambda}\sqrt{\lambda+\delta} <2\lambda+\delta-2\lambda = \delta.
    \]
\end{proof}

\begin{proof}[Proof of Proposition~\ref{prop:diffusion_param_bound}]
Combine Lemmas~\ref{lem:diffusion_init_error}, \ref{lem:diffusion_discretization_error} and~\ref{lem:diffusion_es_error} through the triangle inequality \eqref{eq:diffusion_decomp_W2}.
\end{proof}

\subsection{Choice of hyperparameters}
\label{app:diffusion_hyperparameters}

Thanks to Proposition~\ref{prop:diffusion_param_bound}, we can now prove the non-asymptotic bounds of Theorem~\ref{thm:diffusion_spectral_bound} (including the properties of $(a_N)$ stated there) and of Proposition~\ref{prop:diffusion_coarse_bound}.

\begin{proof}[Proof of Theorem~\ref{thm:diffusion_spectral_bound}, non-asymptotic bound]
    Let $N\ge1$ and recall the hyperparameters of Theorem~\ref{thm:diffusion_spectral_bound}:
    \[
    T = \left(\frac{\Tr(\Sigma^5)}{\Tr\Sigma}\right)^{1/4}\sqrt{N},\quad\delta = \frac{\Tr\Sigma}{dN^2}.
    \]
    We first check that $T>\delta$, as required by Proposition~\ref{prop:diffusion_param_bound}. By H\"older's inequality, $\Tr(\Sigma^5)\ge(\Tr\Sigma)^5/d^4$, hence
    \[
    T\ge\left(\frac{(\Tr\Sigma)^4}{d^4}\right)^{1/4}\sqrt{N} = \frac{\Tr\Sigma}d\sqrt{N},
    \]
    which exceeds $\delta = \frac{\Tr\Sigma}{d}N^{-2}$ as soon as $N\ge2$. For $N=1$, the inequality $T>\delta$ still holds unless H\"older's inequality is an equality, \textit{i.e.}, unless $\Sigma$ is a scalar matrix. In that isotropic case one has $T=\delta$, the algorithm performs a single step of size $\eta_1=0$, and $\P_1 = \Nml(0,T\I) = \P$, so that $W_2(\P_1,\P)=0$ and the bound holds trivially. This degenerate case is of no interest, but covering it allows Theorem~\ref{thm:diffusion_spectral_bound} to be stated for all $N\ge1$. We assume $T>\delta$ from now on.

    Proposition~\ref{prop:diffusion_param_bound} then gives
    \begin{equation}
    \label{eq:diffusion_majo_proposition_bis}
    W_2(\P_N,\P)\le\frac{\sqrt{\Tr((\Sigma+\delta\I)^3\Sigma^2)}}{T^2} + \frac{\sqrt{\Tr\Sigma+4d\delta}}{4N}\Big(\frac{T}{\delta}\Big)^{1/N}\log\frac{T}{\delta} + \sqrt{d\delta},
    \end{equation}
    and we bound the three terms in turn.

    \emph{Initialization.} With the value of $T$,
    \begin{equation}
    \label{eq:diffusion_majo_init_HP}
    \frac{\sqrt{\Tr((\Sigma+\delta\I)^3\Sigma^2)}}{T^2} = \frac{\sqrt{\Tr\Sigma}}{N} \times \sqrt{\frac{\Tr((\Sigma+\delta\I)^3\Sigma^2)}{\Tr(\Sigma^5)}}.
    \end{equation}
    Expanding the binomial and inserting the value of $\delta$,
    \[
    \Tr((\Sigma+\delta\I)^3\Sigma^2) = \sum_{j=0}^3 \binom{3}{j}\delta^j\Tr(\Sigma^{5-j}) = \sum_{j=0}^3 \binom{3}{j} \left(\frac{\Tr\Sigma}{dN^2}\right)^j \Tr(\Sigma^{5-j}).
    \]
    By H\"older's inequality again, $\Tr(\Sigma^k) \leq d^{1-k/5}\Tr(\Sigma^5)^{k/5}$ for $1\le k\le5$; applied to $k=1$ and $k=5-j$, this gives $(\Tr\Sigma/d)^j\le d^{-j/5}\Tr(\Sigma^5)^{j/5}$ and $\Tr(\Sigma^{5-j})\le d^{j/5}\Tr(\Sigma^5)^{1-j/5}$, so the powers of $d$ cancel and
    \[
    \Tr((\Sigma+\delta\I)^3\Sigma^2) \leq \sum_{j=0}^3 \binom{3}{j} \frac{1}{N^{2j}}\Tr(\Sigma^5) = \left(1+\frac{1}{N^2}\right)^3\Tr(\Sigma^5).
    \]
    Returning to \eqref{eq:diffusion_majo_init_HP},
    \begin{equation*}
    \frac{\sqrt{\Tr((\Sigma+\delta\I)^3\Sigma^2)}}{T^2} \le \bigg(1+\frac1{N^2}\bigg)^{3/2}\frac{\sqrt{\Tr\Sigma}}{N} \leq \sqrt8\,\frac{\sqrt{\Tr\Sigma}}{N}.
    \end{equation*}

    \emph{Discretization.} On the one hand, $\Tr\Sigma+4d\delta = (1+4N^{-2})\Tr\Sigma$. On the other hand,
    \[
    \frac T\delta = \left(\frac{\Tr(\Sigma^5)}{\Tr\Sigma}\right)^{1/4}\frac{d N^{5/2}}{\Tr\Sigma}\le\frac{d\lambda_{\max}}{\Tr\Sigma}N^{5/2} = \rho N^{5/2},
    \]
    where we used $\Tr(\Sigma^5)\le\lambda_{\max}^4\Tr\Sigma$, and where $\rho := d\lambda_{\max}/\Tr\Sigma\in[1,d]$ is the ratio between the largest eigenvalue of $\Sigma$ and the arithmetic mean of its spectrum: the more anisotropic $\Sigma$, the larger $\rho$. Since $x\mapsto x^{1/N}\log x$ is increasing on $[1,\infty)$ and $T/\delta>1$, the discretization term is at most
    \[
    \frac{(1+4N^{-2})^{1/2}\sqrt{\Tr\Sigma}}{4N}\,(\rho N^{5/2})^{1/N}\log(\rho N^{5/2}) = \frac58\,a_N\,\frac{\sqrt{\Tr\Sigma}}{N}\,\rho^{1/N}\log(\rho^{2/5}N),
    \]
    where we wrote $\log(\rho N^{5/2}) = \frac52\log(\rho^{2/5}N)$ and $a_N := (1+4N^{-2})^{1/2}N^{5/(2N)}$ as in the statement.

    \emph{Early stopping.} Immediately, $\sqrt{d\delta} = \sqrt{\Tr\Sigma}/N$.

    Summing the three bounds in \eqref{eq:diffusion_majo_proposition_bis} gives the bound of Theorem~\ref{thm:diffusion_spectral_bound},
    \[
    W_2(\P_N,\P) \le \frac{\sqrt{\Tr\Sigma}}{N}\Big(\frac58\,a_N\,\rho^{1/N}\log(\rho^{2/5}N)+1+\sqrt{8}\Big).
    \]

    \emph{Properties of $(a_N)$.} Both factors of $a_N$ tend to $1$, so $a_N\to1$. Moreover, $(1+4N^{-2})^{1/2}$ is decreasing in $N$, and $N^{1/N}$ is decreasing for $N\ge3$ (as the study of $x\mapsto x^{1/x}$ shows), so $(a_N)_{N\ge3}$ is decreasing and
    \begin{equation}
    \label{eq:diffusion_bound_aN}
    \sup_{N\ge1}a_N = \max\big(a_1,a_2,a_3\big) = \max\Big(\sqrt5,\ \sqrt2\times2^{5/4},\ \tfrac13\sqrt{13}\times3^{5/6}\Big) = 2^{7/4}.
    \end{equation}
    In particular $\frac58a_N\le5/2^{5/4}$, which yields a version of the bound with fixed numerical constants.
\end{proof}

We now turn to Proposition~\ref{prop:diffusion_coarse_bound}, which we restate for convenience.

\DiffusionRatesCoarse*

\begin{proof}
Let $N\ge1$ and set $T=\lambda_{\max}\sqrt N$ and $\delta = \lambda_{\max}/N^2$ as in the statement. The condition $T>\delta$ of Proposition~\ref{prop:diffusion_param_bound} holds as soon as $N\ge2$.
For $N=1$, we have $T=\delta$; as in the proof of Theorem~\ref{thm:diffusion_spectral_bound}, the algorithm performs a single step of size $\eta_1=0$, and $\P_1 = \Nml(0,\lambda_{\max}\I)$. Hence, $W_2(\P_1,\P)\le\sqrt{d\lambda_{\max}}$ and the bound of Proposition~\ref{prop:diffusion_coarse_bound} holds. As before, this degenerate case only serves to state the proposition for all $N\ge1$.

From now on, we assume $N\ge2$, so that $T>\delta$ and Proposition~\ref{prop:diffusion_param_bound} applies:
\begin{equation}
\label{eq:diffusion_majo_proposition_ter}
    W_2(\P_N,\P)\le\frac{\sqrt{\Tr((\Sigma+\delta\I)^3\Sigma^2)}}{T^2} + \frac{\sqrt{\Tr\Sigma+4d\delta}}{4N}\Big(\frac{T}{\delta}\Big)^{1/N}\log\frac{T}{\delta} + \sqrt{d\delta}.
\end{equation}
We bound the three terms in turn.

\emph{Initialization.} Bounding each eigenvalue $\lambda_i$ by $\lambda_{\max}$, we have
\[
\frac{\sqrt{\Tr((\Sigma+\delta\I)^3\Sigma^2)}}{T^2} \le \frac{\sqrt{(1+N^{-2})^3d\lambda_{\max}^5}}{\lambda_{\max}^2N} =(1+N^{-2})^{3/2}\frac{\sqrt{d\lambda_{\max}}}N.
\]
Since $N\ge2$, we get
\[
\frac{\sqrt{\Tr((\Sigma+\delta\I)^3\Sigma^2)}}{T^2} \le \Big(\frac54\Big)^{3/2}\frac{\sqrt{d\lambda_{\max}}}{N}.
\]

\emph{Discretization.} Since $T/\delta = N^{5/2}$, using $\Tr\Sigma\le d\lambda_{\max}$ gives
\begin{align*}
\frac{\sqrt{\Tr\Sigma+4d\delta}}{4N}\Big(\frac{T}{\delta}\Big)^{1/N}\log\frac{T}{\delta} &\le \frac{\sqrt{d\lambda_{\max}(1+4N^{-2})}}{4N}N^{5/(2N)}\log(N^{5/2})\\ &= \frac58(1+4N^{-2})^{1/2}N^{5/(2N)}\sqrt{d\lambda_{\max}}\frac{\log N}N \\
&=\frac58a_N\sqrt{d\lambda_{\max}}\frac{\log N}{N} \\
&\le \frac5{2^{5/4}}\sqrt{d\lambda_{\max}}\frac{\log N}N,
\end{align*}
where the last inequality is \eqref{eq:diffusion_bound_aN}.

\emph{Early stopping.} Immediately, $\sqrt{d\delta} = \sqrt{d\lambda_{\max}}/N$.

Summing the three bounds in \eqref{eq:diffusion_majo_proposition_ter} then yields the announced bound.
\end{proof}

The equivalent \eqref{eq:spectral_equivalent}, on which the asymptotic exactness claimed in Theorem~\ref{thm:diffusion_spectral_bound} rests, is the object of the next section.

\subsection{Asymptotic equivalent}
\label{app:diffusion_equivalent_asymptotique}
We prove the following proposition, of which the asymptotic equivalent of Theorem~\ref{thm:diffusion_spectral_bound} is a special case. It also covers the hyperparameters of Proposition~\ref{prop:diffusion_coarse_bound}, namely $A=B=\lambda_{\max}$, so that the asymptotic equivalent \eqref{eq:spectral_equivalent} holds in that case as well.

\begin{proposition}
\label{prop:diffusion_equiv_asym_general}
    Let $A,B>0$ be constants independent of $N$. With hyperparameters of the form
    \[
    T = A\sqrt{N},\quad\delta = \frac B{N^2},
    \]
    and the geometric discretization $\tau_k = T(\delta/T)^{k/N}$, the error of the diffusion algorithm satisfies
    \[
    W_2(\P_N,\P) \underset{N\to\infty}{\sim}\frac58\sqrt{\Tr\Sigma}\,\frac{\log N}{N}.
    \]
\end{proposition}

\begin{proof}
Let $T = A\sqrt N$, $\delta=BN^{-2}$ and $\tau_k = T(\delta/T)^{k/N}$ as in the statement, and take $N$ large enough that $T>\delta$. By Lemma~\ref{lem:W2_gaussian},
\begin{equation}
\label{eq:diffusion_equiv_W2_vpr}
W_2^2(\P_N,\P) = \sum_{i=1}^d \Big(\sqrt{\li^{(N)}}-\sqrt{\li}\Big)^2
= \sum_{i=1}^d \Bigg(\frac{\li^{(N)}-\li}{\sqrt{\li^{(N)}}+\sqrt{\li}}\Bigg)^2,
\end{equation}
so we need an equivalent of $\lambda^{(N)}-\lambda$, which we split as
\[
\lambda^{(N)}-\lambda = (\lambda^{(N)}-\tilde\lambda^{(N)})+(\tilde\lambda^{(N)}-(\lambda+\delta))+\delta.
\]
By \eqref{eq:diffusion_init_error_diff_vpr}, the first term is $-\big(\frac{\lambda+\delta}{\lambda+T}\big)^2\lambda = O(T^{-2}) = O(N^{-1})$, and the third is $\delta = O(N^{-2})$. Hence
\begin{equation}
\label{eq:diffusion_equiv_diff_vpr_1}
\lambda^{(N)}-\lambda = \tilde\lambda^{(N)}-(\lambda+\delta) + O\left(\frac1N\right),
\end{equation}
and by \eqref{eq:diffusion_initialization_error_comp_vpr} the remaining term is
\begin{equation}
\label{eq:diffusion_equiv_diff_vpr_2}
\tilde\lambda^{(N)}-(\lambda+\delta) = (\lambda+\delta)^2\Delta.
\end{equation}
Everything thus reduces to an equivalent of the quadrature error $\Delta$, which we obtain by squeezing it between two quantities with the same asymptotics.

\emph{Upper bound.} By \eqref{eq:diffusion_discr_error_majo_delta_integrale}, since the integrand is positive,
\begin{align*}
\Delta &\le (q^{-1}-1)\int_\delta^T \frac{\tau\dd\tau}{(\lambda+\tau)^3}\le(q^{-1}-1)\int_0^\infty \frac{\tau\dd\tau}{(\lambda+\tau)^3} \\
&= (q^{-1}-1)\left[\frac\lambda{2(\lambda+\tau)^2} - \frac1{\lambda+\tau}\right]_0^\infty = \frac{q^{-1}-1}{2\lambda},
\end{align*}
and, since $\frac1N\log\frac{AN^{5/2}}{B}\to0$,
\begin{equation}
\label{eq:diffusion_DL}
q^{-1}-1 = \left(\frac {AN^{5/2}}B\right)^{1/N}-1 = \exp\left(\frac1N\log\frac{AN^{5/2}}B\right)-1\sim\frac1N\log\frac{AN^{5/2}}B\sim\frac52\frac{\log N}N.
\end{equation}

\emph{Lower bound.} Since $q\le1$ and $\tau_j = q\tau_{j-1}$, for every $j\ge1$,
\[
f(\tau_j) = \frac{q\tau_{j-1}}{\lambda+q\tau_{j-1}}\ge\frac{q\tau_{j-1}}{\lambda+\tau_{j-1}} = qf(\tau_{j-1}),
\]
where $f:\tau\mapsto\frac{\tau}{\lambda+\tau}$ is the increasing function of Section~\ref{app:diffusion_param_bound}. Returning to \eqref{eq:diffusion_discr_error_delta_inter_0} and using $(q^{-1}-1)q = 1-q$,
\begin{align*}
\Delta &\ge (1-q)\sum_{j=1}^Nf(\tau_{j-1})\left(\frac1{\lambda+\tau_j}-\frac1{\lambda+\tau_{j-1}}\right)
= (1-q)\sum_{j=1}^Nf(\tau_{j-1})\int_{\tau_{j}}^{\tau_{j-1}} \frac{\dd\tau}{(\lambda+\tau)^2} \\
&\ge (1-q)\sum_{j=1}^N\int_{\tau_j}^{\tau_{j-1}} \frac{f(\tau)\dd\tau}{(\lambda+\tau)^2}
=(1-q)\int_\delta^T \frac{\tau\dd\tau}{(\lambda+\tau)^3},
\end{align*}
the second inequality using $f(\tau)\le f(\tau_{j-1})$ on $[\tau_j,\tau_{j-1}]$. As $\delta\to0$ and $T\to\infty$, the integrability of $\tau\mapsto\tau(\lambda+\tau)^{-3}$ gives
\[
\int_\delta^T \frac{\tau\dd\tau}{(\lambda+\tau)^3} \underset{N\to\infty}{\longrightarrow} \int_0^\infty \frac{\tau\dd\tau}{(\lambda+\tau)^3} = \frac1{2\lambda},
\]
while $q\to1$ and \eqref{eq:diffusion_DL} give $1-q = q(q^{-1}-1)\sim\frac52\frac{\log N}N$.

\emph{Conclusion.} We have shown
\[
\underbrace{(1-q)\int_\delta^T \frac{\tau\dd\tau}{(\lambda+\tau)^3}}_{\sim\frac5{4\lambda}\frac{\log N}N} \le \Delta \le \underbrace{\frac{q^{-1}-1}{2\lambda}}_{\sim\frac5{4\lambda}\frac{\log N}N},
\quad\text{hence}\quad
\Delta \underset{N\to\infty}{\sim} \frac{5}{4\lambda}\frac{\log N}N.
\]
Plugging this into \eqref{eq:diffusion_equiv_diff_vpr_2} and then \eqref{eq:diffusion_equiv_diff_vpr_1}, and using $\delta\to0$,
\[
\lambda^{(N)}-\lambda\sim\tilde\lambda^{(N)}-(\lambda+\delta) \sim \frac{5\lambda}4\frac{\log N}N.
\]
In particular $\lambda^{(N)}\to\lambda$, so the denominators in \eqref{eq:diffusion_equiv_W2_vpr} tend to $2\sqrt\lambda$ and
\[
W_2^2(\P_N,\P) \sim \sum_{i=1}^d \frac{(\li^{(N)}-\li)^2}{4\li} \sim \bigg(\frac58\frac{\log N}N\bigg)^2\sum_{i=1}^d\li = \bigg(\frac58\frac{\log N}N\bigg)^2\Tr\Sigma,
\]
which is the announced equivalent.
\end{proof}

\begin{proof}[Proof of Theorem~\ref{thm:diffusion_spectral_bound}, asymptotic equivalent and exactness]
The hyperparameters of Theorem~\ref{thm:diffusion_spectral_bound} are of the form of Proposition~\ref{prop:diffusion_equiv_asym_general} with $A = (\Tr(\Sigma^5)/\Tr\Sigma)^{1/4}$ and $B = \Tr\Sigma/d$, which gives the equivalent \eqref{eq:spectral_equivalent}. Dividing the non-asymptotic bound by $\frac58\sqrt{\Tr\Sigma}\frac{\log N}{N}$, the ratio of the bound to the error is asymptotically
\[
a_N\,\rho^{1/N}\,\frac{\log(\rho^{2/5}N)}{\log N}+\frac{8(1+\sqrt8)}{5\log N}
= a_N\,\rho^{1/N}\bigg(1+\frac{2\log\rho}{5\log N}\bigg)+\frac{8(1+\sqrt8)}{5\log N}
\underset{N\to\infty}{\longrightarrow}1
\]
for fixed $\Sigma$, since $a_N\to1$, $\rho^{1/N}\to1$ and $\rho$ does not depend on $N$. The bound is therefore asymptotically exact, which completes the proof of Theorem~\ref{thm:diffusion_spectral_bound}.
\end{proof}

\section{Proof of Corollary~\ref{cor:complexity} (\texorpdfstring{$\bm{\eps}$}{epsilon}-complexity)}
\label{app:cor_complexity_proof}
This section proves the following corollary, stated in Section~\ref{results_diffusion}.

\complexity*

The three $\eps$-complexities follow from the parametric bounds of Propositions~\ref{prop:ULA_uniform_bound}, \ref{prop:ULD_uniform_bound} and~\ref{prop:diffusion_param_bound} respectively.

\begin{proof}
Let $\eps\le\sqrt{\Tr\Sigma}$. In each case we choose the hyperparameters and the budget $N$ so that the error terms are all at most a fixed fraction of $\eps$.

\paragraph{ULA.}
By Proposition~\ref{prop:ULA_uniform_bound}, for every $0<\eta\le\lambda_{\min}$ and every $N\ge0$,
\begin{equation}
\label{eq:complexity_ULA}
    W_2(\P_N,\P)\le \sqrt{2\Tr\Sigma}\,e^{-2\eta N/\lambda_{\max}}+\frac12\sqrt{\Tr(\Sigma^{-1})}\,\eta,
\end{equation}
and we make each term at most $\eps/2$. For the discretization term,
\begin{equation}
\label{eq:complexite_ULA_discr}
\frac12\sqrt{\Tr(\Sigma^{-1})}\,\eta\le\frac\eps2
\iff
\eta\le\frac\eps{\sqrt{\Tr(\Sigma^{-1})}}.
\end{equation}
We take
\begin{equation}
\label{eq:complexite_ULA_eta}
\eta=\min\bigg(\lambda_{\min},\,\frac\eps{\sqrt{\Tr(\Sigma^{-1})}}\bigg),
\end{equation}
which is admissible for Proposition~\ref{prop:ULA_uniform_bound} and saturates \eqref{eq:complexite_ULA_discr} whenever possible: it is the largest admissible step-size, hence the one that makes the exponential factor decay fastest. For the early-stopping term,
\[
\sqrt{2\Tr\Sigma}\,e^{-2\eta N/\lambda_{\max}}\le\frac\eps2
\iff
N\ge\frac{\lambda_{\max}}{2\eta}\log\frac{\sqrt{8\Tr\Sigma}}{\eps},
\]
and with \eqref{eq:complexite_ULA_eta},
\[
\frac{\lambda_{\max}}{2\eta}
=\frac{\lambda_{\max}}{2}\max\bigg(\frac{1}{\lambda_{\min}},\frac{\sqrt {\Tr(\Sigma^{-1})}}{\eps}\bigg)
=\frac\kappa2\max\bigg(1,\frac{\lambda_{\min}\sqrt{\Tr(\Sigma^{-1})}}{\eps}\bigg).
\]
This is the announced condition
\[
N\ge\frac\kappa2\max\bigg(1,\frac{\lambda_{\min}\sqrt{\Tr(\Sigma^{-1})}}{\eps}\bigg)\log\frac{\sqrt{8\Tr\Sigma}}{\eps}.
\]

\paragraph{ULD.}
By the first bound of Proposition~\ref{prop:ULD_uniform_bound}, for every $\gamma\ge\sqrt{8/\lambda_{\min}}$, every $\eta\le1/\gamma$ and every $N\ge0$,
\[
W_2(\P_N,\P)\leq8\sqrt{\Tr\Sigma}\,e^{-2\eta N/(\lambda_{\max}\gamma)}+\frac67\sqrt{\Tr(\Sigma^{-1})}\,\frac\eta\gamma.
\]
This bound depends on $(\eta,\gamma)$ only through $\phi:=\eta/\gamma$, and the constraints translate into $\phi\le1/\gamma^2\le\lambda_{\min}/8$; conversely, every $\phi\in(0,\lambda_{\min}/8]$ is reached by an admissible pair, for instance $\gamma=\sqrt{8/\lambda_{\min}}$ and $\eta=\phi\gamma$. We therefore work with the bound
\[
W_2(\P_N,\P)\leq8\sqrt{\Tr\Sigma}\,e^{-2\phi N/\lambda_{\max}}+\frac67\sqrt{\Tr(\Sigma^{-1})}\,\phi,\quad\phi\le\frac{\lambda_{\min}}8.
\]
Up to some numerical constants, this bound has the same form as in \eqref{eq:complexity_ULA}. We proceed as before and obtain the condition of Corollary~\ref{cor:complexity}, namely
\[
N\ge4\kappa\max\bigg(1,\frac3{14}\frac{\lambda_{\min}\sqrt{\Tr(\Sigma^{-1})}}{\eps}\bigg)\log\frac{16\sqrt{\Tr\Sigma}}\eps,
\]
with
\[
\eta = \min\bigg(\sqrt{\frac{\lambda_{\min}}8}, \frac{7\sqrt2}6\frac\eps{\sqrt{\lambda_{\min}\Tr(\Sigma^{-1})}}\bigg), \quad \gamma = \sqrt\frac8{\lambda_{\min}}.
\]

\paragraph{Diffusion.}
We show the more general statement of Remark~\ref{rem:variation_complexity}, namely that, with
\[
\delta = \eps^2/(9d),\quad T = \sqrt{3/\eps}\Tr\big((\Sigma+\delta\I)^3\Sigma^2\big)^{1/4}
\]
and the geometric grid $\tau_k = T(\delta/T)^{k/N}$, an error $\eps$ is achieved by taking, for any $r\ge1$,
\[
N\ge r\max\bigg(1,\frac{3e^{1/r}}{4r}\frac{\sqrt{\Tr\Sigma+4\eps^2/9}}\eps\bigg)\log\frac T\delta.
\]
By Proposition~\ref{prop:diffusion_param_bound}, for every $N\ge1$ and every $T>\delta>0$,
\begin{equation*}
W_2(\P_N,\P)\le\frac{\sqrt{\Tr\big((\Sigma+\delta\I)^3\Sigma^2\big)}}{T^2}
+\frac{\sqrt{\Tr\Sigma+4d\delta}}{4N}\Big(\frac T\delta\Big)^{1/N}\log\frac T\delta
+\sqrt{d\delta},
\end{equation*}
and we make each of the three terms at most $\eps/3$.

\emph{Early stopping.} We have
\[
\sqrt{d\delta}\le\frac\eps3 \Longleftrightarrow \delta\le\frac{\eps^2}{9d}.
\]
We saturate this constraint by taking
\begin{equation}
\label{eq:complexite_diffusion_delta}
\delta = \frac{\eps^2}{9d}.
\end{equation}

\emph{Initialization.} We have
\[
\frac{\Tr\big((\Sigma+\delta\I)^3\Sigma^2\big)^{1/2}}{T^2} \le \frac\eps3 \Longleftrightarrow T\ge\sqrt\frac3\eps\Tr\big((\Sigma+\delta\I)^3\Sigma^2\big)^{1/4},
\]
so we take
\begin{equation}
\label{eq:complexite_diffusion_T}
T=\sqrt\frac3\eps\Tr\big((\Sigma+\delta\I)^3\Sigma^2\big)^{1/4}.
\end{equation}
This choice is compatible with the constraint $T>\delta$. Indeed, with \eqref{eq:complexite_diffusion_delta} and \eqref{eq:complexite_diffusion_T},
\begin{equation}
\label{eq:complexity_diffusion_ratio}
\frac T\delta = \Big(\frac3\eps\Big)^{5/2}d\Tr\big((\Sigma+\delta\I)^3\Sigma^2\big)^{1/4} > \Big(\frac3\eps\Big)^{5/2}d\Tr(\Sigma^5)^{1/4}.
\end{equation}
But, by Hölder's inequality, $\Tr(\Sigma^5)\ge(\Tr\Sigma)^5/d^4$, so \eqref{eq:complexity_diffusion_ratio} yields
\begin{equation}
\label{eq:complexity_diffusion_ratio_bis}
\frac T\delta > 3^{5/2}\left(\frac{\Tr\Sigma}{\eps^2}\right)^{5/4}>1,
\end{equation}
since we assumed $\eps\le\sqrt{\Tr\Sigma}$.

\emph{Discretization.} It remains to ensure
\[
E:=\frac{\sqrt{\Tr\Sigma+4d\delta}}{4N}\Big(\frac T\delta\Big)^{1/N}\log\frac T\delta\le\frac\eps3.
\]
The factor $(T/\delta)^{1/N}$ is what prevents solving $E\le\eps/3$ in closed form; but it tends to 1 as $N$ grows, and it is under control as soon as $N$ exceeds a few multiples of $\log(T/\delta)$. Precisely, for any $r\ge1$,
\begin{equation}
\label{eq:complexity_diffusion_floor}
    \Big(\frac T\delta\Big)^{1/N}\le e^{1/r} \Longleftrightarrow N\ge r\log\frac T\delta.
\end{equation}
Under \eqref{eq:complexity_diffusion_floor} for a given $r\ge1$,
\[
E\le \frac{e^{1/r}}4\frac{\sqrt{\Tr\Sigma+4d\delta}}N\log\frac T\delta,
\]
and
\[
\frac{e^{1/r}}4\frac{\sqrt{\Tr\Sigma+4d\delta}}N\log\frac T\delta \le \frac\eps3 \Longleftrightarrow N\ge \frac{3e^{1/r}}4\frac{\sqrt{\Tr\Sigma+4d\delta}}\eps\log\frac T\delta.
\]
Both requirements are met as soon as
\begin{equation}
\label{eq:complexity_diffusion_max}
N\ge\max\bigg(r\log\frac T\delta,\frac{3e^{1/r}}4\frac{\sqrt{\Tr\Sigma+4d\delta}}\eps\log\frac T\delta\bigg) = r\max\bigg(1,\frac{3e^{1/r}}{4r}\frac{\sqrt{\Tr\Sigma+4\eps^2/9}}\eps\bigg)\log\frac T\delta,
\end{equation}
which is the budget of Remark~\ref{rem:variation_complexity}. The parameter $r$ tunes a trade-off: a larger $r$ brings the constant $3e^{1/r}/4$ down towards $3/4$, at the price of a higher floor $r\log(T/\delta)$.

Taking $r=1$, since $\sqrt{\Tr\Sigma + 4\eps^2/9}\ge \frac23\eps$, the second argument of the maximum in \eqref{eq:complexity_diffusion_max} is at least $\frac{3e}4\times\frac23=\frac e2>1$. The maximum is therefore always attained by the second argument, and \eqref{eq:complexity_diffusion_max} reduces to
\begin{equation}
\label{eq:complexity_diffusion_N1}
N\ge\frac{3e}4\frac{\sqrt{\Tr\Sigma+4\eps^2/9}}{\eps}\log \frac T\delta,
\end{equation}
the budget of Corollary~\ref{cor:complexity}.
\end{proof}

\begin{remark}[Order of magnitude of $\log(T/\delta)$]
    We can make the logarithmic factor in \eqref{eq:complexity_diffusion_max} explicit. By \eqref{eq:complexity_diffusion_ratio_bis}, we have the lower bound
    \[
    \log\frac T\delta \ge \frac52\log\frac{3\sqrt{\Tr\Sigma}}\eps.
    \]
    For the upper bound, with \eqref{eq:complexite_diffusion_delta} we have
    \[
    \delta \le \frac{\Tr\Sigma}{9d}\le\frac19\lambda_{\max},
    \]
    whence
    \[
    \Tr((\Sigma+\delta\I)^3\Sigma^2)\le(\lambda_{\max}+\delta)^3\lambda_{\max}\Tr\Sigma\le\Big(\frac{10}9\Big)^3\lambda_{\max}^4\Tr\Sigma.
    \]
    With $\rho=d\lambda_{\max}/\Tr\Sigma\in[1,d]$ as in Theorem~\ref{thm:diffusion_spectral_bound}, \eqref{eq:complexity_diffusion_ratio} yields
    \[
    \frac T\delta\le\Big(\frac{10}9\Big)^{3/4}\Big(\frac3\eps\Big)^{5/2}d\lambda_{\max}(\Tr\Sigma)^{1/4} = 3^{5/2}\Big(\frac{10}9\Big)^{3/4}\rho\,\Big(\frac{\sqrt{\Tr\Sigma}}\eps\Big)^{5/2}.
    \]
    Since $\log\big(3^{5/2}(10/9)^{3/4}\big)\approx2.83<3$, we finally obtain
    \begin{equation}
    \label{eq:complexity_diffusion_magnitude_log}
    \frac52\log\frac{3\sqrt{\Tr\Sigma}}\eps\ \le\ \log\frac T\delta\ \le\ \frac52\log\frac{\sqrt{\Tr\Sigma}}\eps+\log\rho+3.
    \end{equation}
In particular, $\log(T/\delta)=\Theta(\log(\sqrt{\Tr\Sigma}/\eps))$ for fixed $\Sigma$. Combined with $\sqrt{\Tr\Sigma+4\eps^2/9}\le\frac{\sqrt{13}}3\sqrt{\Tr\Sigma}$, equation~\eqref{eq:complexity_diffusion_N1} and the upper bound in \eqref{eq:complexity_diffusion_magnitude_log} lead to the fully explicit budget condition
\[
N\ge \frac{\sqrt{13}e}{4}\,\frac{\sqrt{\Tr\Sigma}}\eps\bigg(\frac52\log\frac{\sqrt{\Tr\Sigma}}\eps+\log\rho+3\bigg),
\]
of order $\tilde{O}(\sqrt{\Tr\Sigma}/\eps)$ with no dependence on $\kappa$.
\end{remark}

\section{Proof of Theorem~\ref{thm:implicit_bias} (score learning)}

This section proves the following theorem, stated in Section~\ref{learning}.

\LearningScore*

The proof is a direct computation: we first write both losses in closed form and derive their gradients, then compute the iterates $A_K$ and $A^\sigma_K$ explicitly, and finally invert these expressions.

\subsection{Losses and gradients}
We treat the two losses simultaneously by introducing, for $t\ge0$,
\[
L^{(t)}(A) = \frac1m\sum_{i=1}^m\E_\omega[\nabla\cdot s_A(x_i+\omega)+\frac12\Norm{s_A(x_i+\omega)}^2],
\]
with $\omega \sim \Nml(0,t\I)$. Hence, $L = L^{(0)}$ and $L_\sigma=L^{(\sigma^2)}$. We first compute the losses and their gradients:
\begin{lemma}
\label{lem:grad_loss}
    For all $t\ge0$ and $A\in \R^{d\times d}$, we have
    \[
    L^{(t)}(A) = \frac12\Tr(A\hat\Sigma A^\top)-\Tr A + \frac t2\Norm{A}_F^2,
    \]
    with gradient
    \[
    \nabla L^{(t)}(A) = A(\hat\Sigma+t\I)-\I.
    \]
\end{lemma}

\begin{proof}
    Fix $t\ge0$, $A\in\R^{d\times d}$. For any $x\in\R^d$, we have $s_A(x) = -Ax$, hence
    \begin{align*}
    L^{(t)}(A) &= \frac1m\sum_{i=1}^m\E_\omega[-\Tr A+\frac12\Norm{A(x_i+\omega)}^2] \\
    &= -\Tr A + \frac1{2m}\sum_{i=1}^m\E_\omega[\Norm{Ax_i}^2+2x_i^\top A^\top A\omega + \Norm{A\omega}^2] \\
    &= -\Tr A+\frac12\E_\omega[\omega^\top A^\top A\omega]+\frac1{2m}\sum_{i=1}^m\Tr(Ax_ix_i^\top A^\top)+\frac1m \bigg(\sum_{i=1}^mx_i^\top A^\top A\bigg)\E\omega \\
    &= -\Tr A + \frac12\Tr(t\I AA^\top)+\frac12\Tr(A\hat\Sigma A^\top) \\
    &= \frac12\Tr(A\hat\Sigma A^\top) - \Tr A + \frac t2\Norm{A}_F^2.
    \end{align*}
    Denoting
   \[
   f:A\longmapsto\frac12\Tr(A\hat\Sigma A^\top),\quad g:A\longmapsto -\Tr A,\quad h:A\longmapsto \frac 12\Norm{A}_F^2,
   \]
   we compute, using the symmetry of $\hat\Sigma$,
   \[
   \nabla f(A) = A\hat\Sigma,\quad\nabla g(A) = -\I,\quad\nabla h(A)=A.
   \]
   The gradient of $L^{(t)}$ then follows:
   \[
   \nabla L^{(t)}(A) = A(\hat\Sigma+t\I)-\I.
   \]
    
\end{proof}
Note that in particular, we retrieve the fact that
\[
\forall A,\quad L_\sigma(A) = L(A)+\frac{\sigma^2}2\Norm{A}_F^2,
\]
so the loss $L_\sigma$ is indeed a ridge-type regularization of $L$, as claimed in Section~\ref{learning}.

\subsection{Closed form of the iterates}
Let $\hat\Sigma_t = \hat\Sigma+t\I$ and denote by $(\hat\lambda_i)_{1\le i\le d}$ the eigenvalues of $\hat\Sigma$, with an orthonormal eigenbasis $(u_i)$; $\hat\Sigma_t$ has eigenvalues $\hat\lambda_i+t$ in the same basis. Moreover, denote by $\hat\Sigma^+$ the Moore--Penrose inverse of $\hat\Sigma$, \textit{i.e.}, the matrix defined by
\[
\hat\Sigma^+ = \sum_{\substack{i=1\\ \hat\lambda_i \ne 0}}^d \tfrac1{\hat\lambda_i}u_iu_i^\top.
\]

\begin{proposition}
\label{prop:learning_iterates}
    Let $\alpha>0$ be the gradient step-size. For every $K\ge0$, the iterate of gradient descent on $L^{(t)}$ is
    \[
    A^{(t)}_K = \alpha\sum_{j=0}^{K-1}(\I-\alpha\hat\Sigma_t)^j = \sum_{i=1}^d \psi_K(\hat\lambda_i+t)u_iu_i^\top,
    \]
    with
    \[
    \psi_K : x\longmapsto
    \begin{cases}
        \frac{1-(1-\alpha x)^K}x &\text{if } x\ne 0, \\
        K\alpha &\text{otherwise}.
    \end{cases}
    \]
    Hence, for $t>0$ we have
    \begin{equation}
    \label{eq:learning_iterate_A_K^t}
    A_K^{(t)} = \hat\Sigma_t^{-1}\big(\I-(\I-\alpha\hat\Sigma_t)^K\big),
    \end{equation}
    while
    \begin{equation}
    \label{eq:learning_iterate_A_K}
    A_K^{(0)} = \hat\Sigma^+\big(\I-(\I-\alpha\hat\Sigma)^K\big)+K\alpha\Pi_{\Ker\,\hat\Sigma},
    \end{equation}
    with $\hat\Sigma^+ = \hat\Sigma^{-1}$ and $\Pi_{\Ker\,\hat\Sigma} = 0$ when $\hat\Sigma$ is invertible.
\end{proposition}

\begin{proof}
    Let $K\ge0$. The process $(A^{(t)}_j)_{0\le j\le K}$ is defined by the gradient descent on the loss $L^{(t)}$ with step-size $\alpha$, initialized at $A^{(t)}_0=0$:
    \[
    \begin{cases}
        \forall j,\quad A^{(t)}_{j+1} = A^{(t)}_j-\alpha\nabla L^{(t)}(A^{(t)}_j), \\
        A^{(t)}_0 = 0.
    \end{cases}
    \]
    Therefore, by Lemma~\ref{lem:grad_loss}, we have
    \[
    \forall j,\quad A^{(t)}_{j+1} = A_j^{(t)}(\I-\alpha\hat\Sigma_t) + \alpha\I,
    \]
    thus
    \begin{align*}
    A_K^{(t)} &= \alpha\sum_{j=0}^{K-1} (\I-\alpha\hat\Sigma_t)^j \\
    &=\sum_{i=1}^d\bigg(\alpha\sum_{j=0}^{K-1}\big(1-\alpha(\hat\lambda_i+t)\big)^ju_iu_i^\top\bigg).
    \end{align*}
    The equality
    \[
    A_K^{(t)} = \sum_{i=1}^d \psi_K(\hat\lambda_i+t)u_iu_i^\top
    \]
    is then a simple consequence of the identity
    \[
    \sum_{j=0}^{K-1} x^j =
    \begin{cases}
        \frac{1-x^K}{1-x} &\text{if } x\ne1, \\
        K &\text{otherwise},
    \end{cases}
    \]
    and the relations \eqref{eq:learning_iterate_A_K^t} and \eqref{eq:learning_iterate_A_K} follow.
\end{proof}

\subsection{Inverses of the iterates}
We can now prove Theorem~\ref{thm:implicit_bias}.
\begin{proof}[Proof of Theorem~\ref{thm:implicit_bias}]
Fix $K\ge1$. By Proposition~\ref{prop:learning_iterates}, we have for any $t\ge0$, $\alpha>0$:
\begin{equation*}
A^{(t)}_K = \sum_{i=1}^d\psi_K(\hat\lambda_i+t)u_iu_i^\top.
\end{equation*}
In particular, for every $\alpha<\frac2{\hat\lambda_{\max}+t}$, the eigenvalues $\psi_K(\hat\lambda_i+t)$ of $A_K^{(t)}$ are all positive, thus non-zero, and $A_K^{(t)}$ is invertible with
\begin{equation}
    \label{eq:learning_iterate_A_K_inv}
(A_K^{(t)})^{-1} = \sum_{i=1}^d\frac{u_iu_i^\top}{\psi_K(\hat\lambda_i+t)}.
\end{equation}

\paragraph{Loss $L$ ($t=0$).}
We have
\[
\psi_K(\hat\lambda_i) =
\begin{cases}
    \frac{1-(1-\alpha\hat\lambda_i)^K}{\hat\lambda_i} &\text{if }\hat\lambda_i\ne0, \\
    K\alpha &\text{otherwise},
\end{cases}
\]
equation \eqref{eq:learning_iterate_A_K_inv} therefore yields for $t=0$:
\begin{align}
    A_K^{-1} &= \sum_{\substack{i=1\\ \hat\lambda_i \ne 0}}^d \frac{u_iu_i^\top}{\psi_K(\hat\lambda_i)} + \sum_{\substack{i=1 \\ \hat\lambda_i = 0}}^d \frac{u_iu_i^\top}{\psi_K(\hat\lambda_i)} \notag \\
    &= \sum_{\substack{i=1\\ \hat\lambda_i \ne 0}}^d\frac{\hat\lambda_i}{1-(1-\alpha\hat\lambda_i)^K}u_iu_i^\top + \sum_{\substack{i=1 \\ \hat\lambda_i = 0}}^d \frac1{K\alpha}u_iu_i^\top \notag \\
    &= \sum_{\substack{i=1\\ \hat\lambda_i \ne 0}}^d\hat\lambda_iu_iu_i^\top + \sum_{\substack{i=1\\ \hat\lambda_i \ne 0}}^d\frac{\hat\lambda_i(1-\alpha\hat\lambda_i)^K}{1-(1-\alpha\hat\lambda_i)^K}u_iu_i^\top +\frac1{K\alpha}\Pi_{\Ker\,\hat\Sigma} \notag \\
    &= \hat\Sigma+\frac1{K\alpha}\Pi_{\Ker\,\hat\Sigma} + \sum_{\substack{i=1\\ \hat\lambda_i \ne 0}}^d\frac{\hat\lambda_i(1-\alpha\hat\lambda_i)^K}{1-(1-\alpha\hat\lambda_i)^K}u_iu_i^\top. \label{eq:learning_iterate_A_K_inv_bis}
\end{align}
Denote by $R_K$ the last sum. It is a symmetric matrix, diagonal in the basis $(u_i)$, whose eigenvalues are the numbers
\[
r_i:=\frac{\hat\lambda_i(1-\alpha\hat\lambda_i)^K}{1-(1-\alpha\hat\lambda_i)^K}, \quad \hat\lambda_i\ne0.
\]
Since $x\mapsto\abs{1-\alpha x}$ is convex, its maximum over the interval $[\hat\lambda_{\min}^+,\hat\lambda_{\max}]$ is attained at one of the endpoints: $\abs{1-\alpha\hat\lambda_i}\le \max\big(\abs{1-\alpha\hat\lambda_{\min}^+},\abs{1-\alpha\hat\lambda_{\max}}\big)=\rho_1$ for every $i$ with $\hat\lambda_i\ne0$. The assumption $\alpha<2/\hat\lambda_{\max}$ gives $\rho_1<1$; therefore
\[
\abs{r_i}\le\frac{\hat\lambda_{i}\rho_1^K}{1-\rho_1^K}\le\frac{\hat\lambda_i}{1-\rho_1}\rho_1^K.
\]
Hence, denoting $\Norm{\cdot}_2$ the spectral norm, we have
\[
\Norm{R_K}_2\le \frac{\hat\lambda_{\max}}{1-\rho_1}\rho_1^K = O(\rho_1^K),
\]
which, along with \eqref{eq:learning_iterate_A_K_inv_bis}, leads to the first claim of Theorem~\ref{thm:implicit_bias}.

\paragraph{Loss $L_\sigma$ ($t=\sigma^2$).}
The argument is the same, without any kernel to single out; the inverse formula \eqref{eq:learning_iterate_A_K_inv} gives
\begin{align}
(A_K^\sigma)^{-1} &= \sum_{i=1}^d \frac{\hat\lambda_i+\sigma^2}{1-(1-\alpha(\hat\lambda_i+\sigma^2))^K}u_iu_i^\top \notag \\
&= \sum_{i=1}^d \bigg(\hat\lambda_i+\sigma^2+\frac{(\hat\lambda_i+\sigma^2)(1-\alpha(\hat\lambda_i+\sigma^2))^K}{1-(1-\alpha(\hat\lambda_i+\sigma^2))^K}\bigg)u_iu_i^\top \notag \\ 
&= \hat\Sigma+\sigma^2\I + R^\sigma_K, \label{eq:learning_iterate_A_K_sigma_inv}
\end{align}
where
\[
R^\sigma_K = \sum_{i=1}^d\frac{(\hat\lambda_i+\sigma^2)(1-\alpha(\hat\lambda_i+\sigma^2))^K}{1-(1-\alpha(\hat\lambda_i+\sigma^2))^K}u_iu_i^\top.
\]
Under the assumption $\alpha<2/(\hat\lambda_{\max}+\sigma^2)$,
\[
\Norm{R_K^\sigma}_2\le\frac{\hat\lambda_{\max}+\sigma^2}{1-\rho_2}\rho_2^K = O(\rho_2^K),
\]
with $\rho_2=\max\big(\abs{1-\alpha(\hat\lambda_{\min}+\sigma^2)},\abs{1-\alpha(\hat\lambda_{\max}+\sigma^2)}\big)$. With \eqref{eq:learning_iterate_A_K_sigma_inv}, this gives the second claim of Theorem~\ref{thm:implicit_bias}.
\end{proof}

\section*{LLM usage}
\addcontentsline{toc}{section}{LLM usage}
The appendix, initially written by the authors in another language, was translated into English with the help of LLMs. LLMs were also used to assist in checking the proofs and in writing the code that generates the figures. All LLM-assisted content was reviewed and verified by the authors.

\section*{Code availability}
\addcontentsline{toc}{section}{Code availability}
The code used to generate the figures is available at
\begin{center}
    \url{https://github.com/aperbost/diffusion-vs-Langevin}.
\end{center}
The Wasserstein distances are computed in closed form from the expressions of the eigenvalues of the covariance matrices given in the Appendices~\ref{app:ULA_closed_form}, \ref{app:ULD_closed_form} and~\ref{app:diffusion_closed_form}, so no stochastic simulation is involved.

\end{document}